\pdfoutput=1

\documentclass[11pt]{article}

\usepackage{acl}

\usepackage{times}
\usepackage{latexsym}

\usepackage[T1]{fontenc}

\usepackage[utf8]{inputenc}

\usepackage{microtype}

\usepackage{inconsolata}

\usepackage{graphicx}

\usepackage{amsthm}
\usepackage{amsmath}
\usepackage{amsfonts}
\usepackage{multirow}
\usepackage{booktabs}
\usepackage{cleveref}
\newtheorem{theorem}{Theorem}

\newcommand{\TheName}{SGDPO}

\title{\TheName{}: Self-Guided Direct Preference Optimization for Language Model Alignment}

\author{
  \textbf{Wenqiao Zhu \textsuperscript{$\dagger$, 1, 2}},
  \textbf{Ji Liu\textsuperscript{$\dagger$, 1, *}},
  \textbf{Lulu Wang \textsuperscript{1}},
  \textbf{Jun Wu\textsuperscript{1}},
  \textbf{Yulun Zhang\textsuperscript{2}}
\\
  \textsuperscript{1} HiThink Research,
  \textsuperscript{2} Shanghai Jiao Tong University \\
}

\begin{document}
\maketitle
\def\thefootnote{$\dagger$}\footnotetext{Equal Contribution}
\def\thefootnote{*}\footnotetext{Corresponding author: jiliuwork@gmail.com }

\begin{abstract}
Direct Preference Optimization (DPO) is broadly utilized for aligning Large Language Models (LLMs) with human values because of its flexibility. Despite its effectiveness, it has been observed that the capability of DPO to generate human-preferred response is limited and the results of DPO are far from resilient. To address these limitations, in this paper we propose a novel Self-Guided Direct Preference Optimization algorithm, i.e., \TheName{}, which incorporates a \textit{pilot} term to steer the gradient flow during the optimization process, allowing for fine-grained control over the updates of chosen and rejected rewards. We provide a detailed theoretical analysis of our proposed method and elucidate its operational mechanism. Furthermore, we conduct comprehensive experiments on various models and benchmarks. The extensive experimental results demonstrate the consistency between the empirical results and our theoretical analysis and confirm the effectiveness of our proposed approach (up to 9.19\% higher score).
\end{abstract}

\section{Introduction}

Large Language Models (LLMs) pretrained with next-token prediction have experienced rapid advancements \citep{GPT_4o_blog, deepseekai2025deepseekr1incentivizingreasoningcapability,Claude_3_5_Sonnet_blog,Gemini_1_5_blog}. This progress underscores the necessity to align LLM outputs with human values and preferences while safeguarding societal values from harm. Reinforcement Learning from Human Feedback (RLHF) has emerged as a critical method for achieving this alignment and has become an essential component within the LLM training pipeline \citep{Stiennon@LearningFromHF@2020, bai@2022@HF, Bi@deepseek@2024}.

Traditional RLHF typically involves three key steps: Supervised Fine-Tuning (SFT), reward learning, and Reinforcement Learning (RL) optimization. Because the RL optimization step relies heavily on the reward model, it is essential to train a high-quality reward model. However, this necessity adds complexity to the RLHF training process, making it intricate \citep{Ilyas2020A, Engstrom2020Implementation}.
To tackle this issue, Direct Preference Optimization (DPO) removes the need for reward training by reparameterizing the reward model \cite{rafailov2024direct}. Specifically, it maps reward functions to optimal policies by employing the Bradley-Terry Model \citep{Bradley1952RankAO}, thereby transforming preference feedback from online reward models into offline implicit modeling. As a result, DPO simplifies the post-training process.

\begin{figure*}[t]
\centering
\begin{tabular}{cccc}
  \includegraphics[width=0.23\linewidth]{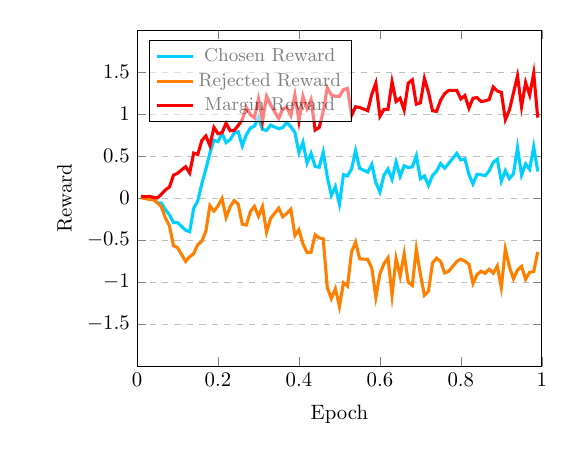} & 
  \includegraphics[width=0.23\linewidth]{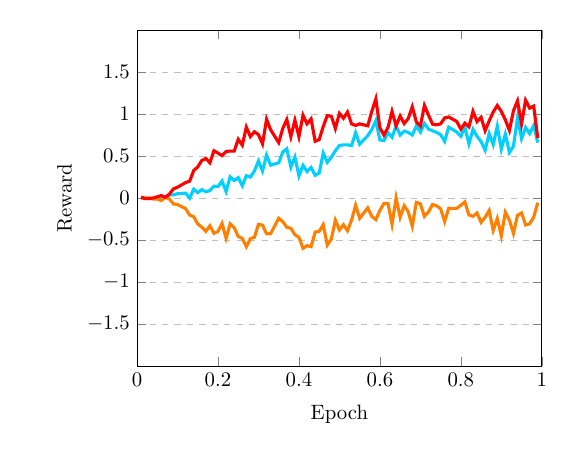} &
  \includegraphics[width=0.23\linewidth]{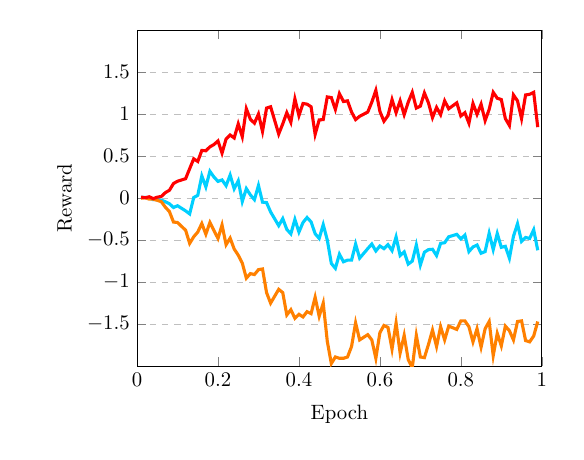} &
  \includegraphics[width=0.23\linewidth]{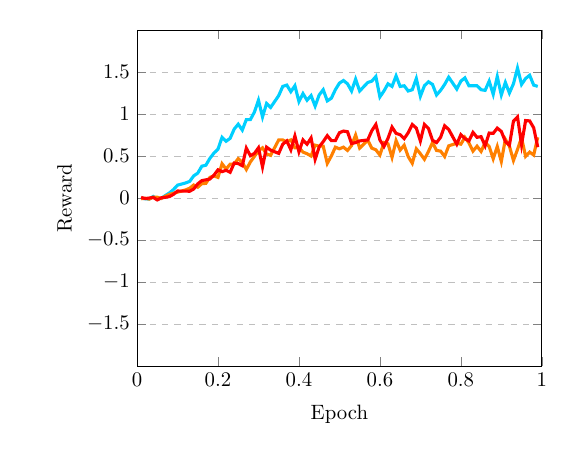} \\
  {\small \hspace{1.8em} (a)} & {\small \hspace{1.8em}(b)} & {\small \hspace{1.8em} (c)} & {\small \hspace{1.8em} (d)} \\
  \includegraphics[width=0.23\linewidth]{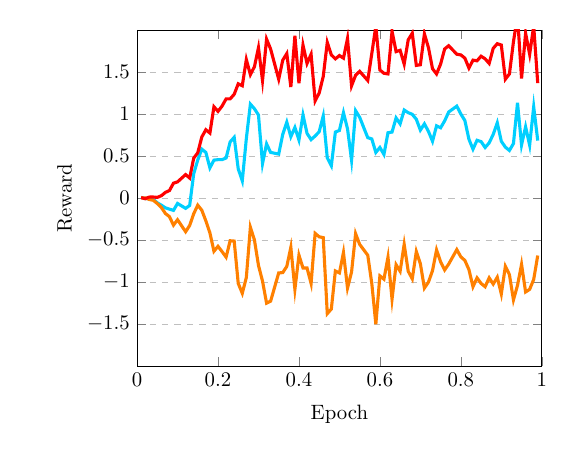} & 
  \includegraphics[width=0.23\linewidth]{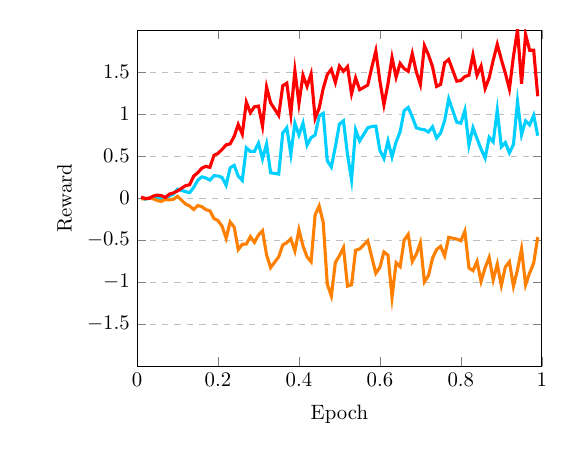} &
  \includegraphics[width=0.23\linewidth]{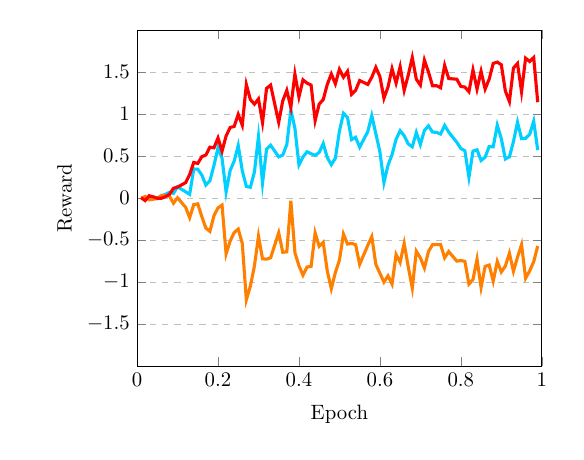} &
  \includegraphics[width=0.23\linewidth]{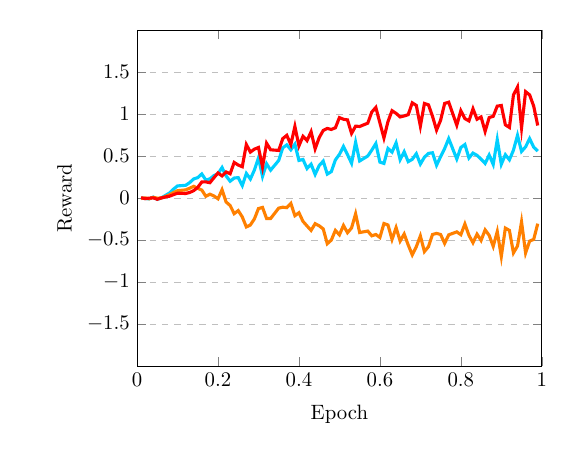} \\
{\small \hspace{1.8em} (e)} &{\small \hspace{1.8em} (f)} & {\small \hspace{1.8em} (g)} & {\small \hspace{1.8em} (h)} \\  
\end{tabular}
\caption{
Reward curves on various base models: (a) DPO reward curves on Llama-3.1 instruct 8B; (b) DPO reward curves on Llama-3.1 base 8B; (c) DPO reward curves on Qwen-2 instruct 7B; (d) DPO reward curves on Qwen-2 base 7B.
(e) Our \TheName{} reward curves on Llama-3.1 instruct 8B; (f) Our \TheName{} reward curves on Llama-3.1 base 8B; (g) Our \TheName{} reward curves on Qwen-2 instruct 7B; (h) Our \TheName{} reward curves on Qwen-2 base 7B.
}
\label{fig:dpo_pilot_cmp}
\end{figure*}

While it has been widely adopted for its flexibility with similar performance levels compared to classic RLHF methods, e.g., PPO \citep{dubois2023alpacafarm}, ChatGLM-RLHF \cite{Hou2024ChatGLMRLHFPO}, the limitations of DPO are observed in a bunch of investigation, which lead to suboptimal alignment performance in LLM training. These limitations include high computational costs \citep{Ethayarajh2024KTOMA,hong-etal-2024-orpo,meng2024simpo}, verbosity \citep{park-etal-2024-disentangling,liu2024lddpo,Lu2024EliminatingBL}, and overfitting \citep{azar2023general,jung2024BC,azar24aipo}. In addition, DPO may still incur inferior capability of LLMs in producing responses that resonate with human preferences \citep{feng2024@dpo_grad_flow}. While LLMs trained with DPO tend to avoid generating responses humans dislike, they struggle to generate responses that humans prefer. 
Furthermore, the efficacy of DPO is inconsistent while being sensitive to the effectiveness of Supervised Fine-Tune (SFT) \citep{feng2024@dpo_grad_flow,xu24cpo}. For instance, LLMs with improper and ineffective settings may lead to poor DPO performance. As illustrated in Figures \ref{fig:dpo_pilot_cmp} (a), (b), (c), and (d), the training reward curves of DPO on various base models using the same preference dataset exhibit extremely high diversity.

Recently, some theoretical works \citep{Pal2024SmaugFF,feng2024@dpo_grad_flow} reveal the reasons behind the limitations of DPO. First, the standard DPO loss can lead to a reduction in the likelihood of preferred examples generated by LLMs \citep{Pal2024SmaugFF}, especially when the Hamming distance between preferred and dispreferred responses is low. Second, the limitations of DPO may be attributed to undesired distinct update patterns in gradient flow between chosen and rejected rewards \citep{feng2024@dpo_grad_flow}. When the optimization process enters an undesired region, the gradient flow of DPO tends to generate an unbalanced update to different variables or incurs difficulties in escaping saddle points, leading to inferior optimization performance.

In this paper, we propose a novel Self-Guided Direct Preference Optimization algorithm, i.e., \TheName{}, to address the aforementioned limitations. We introduce a \textit{pilot} term into the objective function of \TheName{}. This \textit{pilot} term can be adjusted to steer gradient updates towards different regions, resulting in diverse gradient update patterns and consequently leading to distinct optimization processes. In this case, \TheName{} can enhance the alignment capability of LLMs to generate responses preferred by humans, while contributing to the stabilization and resilience of the LLM training process, as well. In addition, we carry out a detailed theoretical analysis to illustrate the robustness and resilience of \TheName{}. Furthermore, we conduct extensive experiments across various models and benchmarks to demonstrate the superb performance of \TheName{}. The major contributions are summarized as follows:
\begin{itemize}
\item We propose a novel preference alignment algorithm, i.e., Self-Guided Direct Preference Optimization (\TheName{}), designed to stabilize the LLM training process and enhance the capability of LLMs so as to generate responses preferred by humans. By incorporating a \textit{pilot} term into the objective function, \TheName{} guides the gradient flow to balanced updates, thereby improving the updates of chosen and rejected rewards.
\item We provide a thorough theoretical analysis of \TheName{}, elucidating its underlying mechanisms for its robustness and resilience. This analysis offers a scheme for controlling the updates of chosen and rejected rewards as well.
\item We conduct extensive experiments across 4 models and 8 benchmarks. Experimental results demonstrate the alignment between our theoretical analysis and empirical observation, which validates the effectiveness of \TheName{}. Specifically, our method has achieved a significant improvement over the DPO method, with the relative increase reaching up to a maximum of 9.19\%.
\end{itemize}

\section{Related Work}

RLHF has been proven effective in aligning LLMs with human values and has seen widespread adoption across various applications, e.g., summarization \citep{Stiennon@LearningFromHF@2020}, safety alignment \citep{bai@2022@HF}, instruction following \citep{ouyang22intructions}, and translation \citep{xu24cpo}. Nevertheless, RLHF requires a complex training pipeline, which has spurred the proposal of DPO \citep{rafailov2024direct} to simplify the LLM training pipeline.

Since the introduction of DPO, a variety of extensions have been proposed to either address its limitations or provide theoretical interpretations. These include new preference optimization techniques \citep{Cal-DPO2024, zeng2024tokenlevel, razin2025unintentional} and analytical studies \citep{Pal2024SmaugFF, feng2024@dpo_grad_flow}. For instance, SimPO \citep{Lu2024EliminatingBL} reduces computational overhead by adopting a reference-free training strategy, while SimPER \citep{simper2025} introduces an inverse perplexity objective to lower the complexity and fine-tuning time of large language models (LLMs). Although SimPER results in a smaller decrease in chosen likelihoods compared to SimPO, it still exhibits a declining trend in chosen likelihoods, indicating limited flexibility. In contrast, our method introduces a mechanism that allows for adjustable control over both chosen and rejected likelihoods, thereby offering greater adaptability.

Similarly, SamPO and LD-DPO \citep{liu2024lddpo} aim to reduce the verbosity often introduced by alignment algorithms due to prior biases in preference data, ultimately improving alignment performance. TDPO \citep{zeng2024tokenlevel} enhances alignment and diversity through a token-level optimization approach. IPO \citep{azar24aipo} mitigates overfitting by introducing a regularization term that pulls the solution toward a reference policy. Cal-DPO \citep{Cal-DPO2024}, on the other hand, improves performance by incorporating absolute reward values instead of relying solely on relative ones—similar in spirit to IPO’s regularization goal. However, in many practical scenarios, exact absolute reward values may not be available, requiring approximations that can lead to suboptimal outcomes. Our method avoids this issue entirely, as it does not depend on absolute reward values and instead provides additional flexibility through tunable parameters for both chosen and rejected likelihoods, as well as their ratios.

Unintentional Unalignment \citep{razin2025unintentional} investigates how similar embeddings from preference data can lead to unintended misalignment. The authors introduce a metric called Centered Hidden Embedding Similarity (CHES) to improve training sample selection. While this approach is promising for dataset curation, our method operates at the optimization level rather than the data level, making it more robust and independent of dataset modifications. Additionally, NCA \citep{chen2024noise} leverages Noise Contrastive Estimation (NCE) to achieve robust alignment, and BCO \citep{Jung2024bco} proposes training a binary classifier where the logit serves as a reward signal, also yielding robust results.

Despite these advances, none of the existing methods effectively tackle both the issue of reduced updates on preferred examples and the challenge of unbalanced updates with difficulties in escaping saddle points simultaneously. Our approach addresses both concerns, offering a more comprehensive and flexible solution to preference-based alignment.

\section{Method}
\subsection{Preliminary of DPO}

DPO is a widely adopted technique for optimizing the preferences of LLMs. This method stands out because of the innovative utilization of an analytical mapping that translates reward functions into optimal policies, streamlining the alignment process without necessitating a direct reward model. The cornerstone of DPO lies in its specific transformation, which can be mathematically formulated by the following equation:
\begin{equation}%
\label{eq:reward}
r(x,y) = \beta \log \frac{\pi_\theta(y|x)}{\pi_\text{ref}(y|x)} + \beta \log Z(x),%
\end{equation}%
where {\(r(x,y)\)} is the reward function, {\(\beta\)} serves as a scaling factor, {\(\pi_\theta(y|x)\)} represents the policy inferred from the reward model, and {\(\pi_\text{ref}(y|x)\)} indicates the reference policy. Here, {\(Z(x)\)} functions as a normalization constant ensuring the probabilities are properly scaled.

By leveraging the Bradley-Terry preference model \cite{bradley1952rank}, DPO expresses the probability that chosen outcome {$y_w$} is preferred over rejected outcome {$y_l$}, given an input prompt instruction {$x$}, as formulated as follows:%
\begin{equation}%
\label{eq:compare}
p(y_w > y_l | x) = \frac{\exp\left(r(x,y_w)\right)}{\exp\left(r(x,y_w)\right) + \exp\left(r(x,y_l)\right)}.%
\end{equation}%
The Formula \ref{eq:compare} quantifies the relative preference between two responses by comparing their associated reward values. Within this probabilistic framework, the loss function of DPO, denoted by { $\mathcal{L}_{DPO}$}, is formulated as Formula \ref{eq:DPOLoss}:
\begin{equation}%
\label{eq:DPOLoss}%
\mathcal{L}_{DPO} = -\mathbb{E}_{(x, y_w, y_l) \sim \mathcal{D}}\left[l_\text{DPO}(\pi_\theta, \pi_\text{ref})\right],%
\end{equation}%
where { \(l_\text{DPO}(\pi_\theta, \pi_\text{ref})\)} is defined by:
\begin{equation}%
\label{eq:ldpo}
l_\text{DPO}(\pi_\theta, \pi_\text{ref}) =  \log \sigma (\Delta)
\end{equation}
Here
\begin{equation}
\Delta  =  \beta \left[\log \frac{\pi_\theta(y_w | x)}{\pi_\text{ref}(y_w |x)}  - \log \frac{\pi_\theta(y_l | x)}{\pi_\text{ref}(y_l |x)}\right]
\end{equation}
\(\sigma\) represents the sigmoid function, and \(\beta\) serves as a scaling factor as that in Formula \ref{eq:reward}. Formulas \ref{eq:DPOLoss} and \ref{eq:ldpo} thereby encapsulate the principles of the Bradley-Terry model, integrating preference data into the learning process. In this way, DPO ensures that the responses of LLMs align with observed human preferences.

\subsection{Optimization Process of DPO}
\label{subsec:DPO}

Given the chosen reward { \(\mathcal{X}_1 = \frac{\pi_\theta(y_w|x)}{\pi_\text{ref}(y_w|x)}\)} and the rejected reward { \(\mathcal{X}_2 = \frac{\pi_\theta(y_l|x)}{\pi_\text{ref}(y_l|x)}\)}, the partial derivatives of { \(l_\text{DPO}\)} with respect to { \(\mathcal{X}_1\)} and { \(\mathcal{X}_2\)} are calculated in Formulas \ref{eq:pd1} and \ref{eq:pd2} \citep{feng2024@dpo_grad_flow}:
\begin{small}
\begin{equation}
\label{eq:pd1}
\frac{\partial l_{\text{DPO}}}{\partial \mathcal{X}_1} = \frac{\beta \mathcal{X}_2^\beta}{\mathcal{X}_1 (\mathcal{X}_1^\beta + \mathcal{X}_2^\beta)},
\end{equation}
\end{small}

\begin{small}
\begin{equation}
\label{eq:pd2}
\frac{\partial l_{\text{DPO}}}{\partial \mathcal{X}_2} = -\frac{\beta \mathcal{X}_2^{\beta - 1}}{\mathcal{X}_1^\beta + \mathcal{X}_2^\beta}.
\end{equation}
\end{small}

Furthermore, the ratio of the increase in the probability of a human-preferred response to the decrease in the probability of a human-dispreferred response is given by:

\begin{small}
\begin{equation}
\left| \frac{\partial l_{\text{DPO}}/\partial \mathcal{X}_1}{\partial l_{\text{DPO}}/\partial \mathcal{X}_2} \right| = \frac{\mathcal{X}_2}{\mathcal{X}_1}
\label{eq:ratio}
\end{equation}
\end{small}

The DPO gradient flow for the chosen and rejected rewards are shown in Figure \ref{fig:dpo_grad}. Based on the theoretical framework outlined above and this figure, we can make the following observations:

\begin{itemize}
\item When {\(\mathcal{X}_2\)} is small, as illustrated in the lower part of Figure \ref{fig:dpo_grad}, the DPO gradient flow tends to decrease {\(\mathcal{X}_2\)} rapidly while making only minor adjustments to {\(\mathcal{X}_1\)}. This behavior limits the ability of LLMs to effectively generate highly preferred responses.
\item As DPO optimization progresses, the chosen reward { \(\mathcal{X}_1\)} increases while the rejected reward { \(\mathcal{X}_2\)} decreases. Consequently, { \(\frac{\mathcal{X}_2}{\mathcal{X}_1} < 1\)}. According to Equation \ref{eq:ratio}, this results in the gradient for the rejected reward being updated more quickly than that for the chosen reward.
\end{itemize}

\begin{figure}[t]
\includegraphics[width=\linewidth]{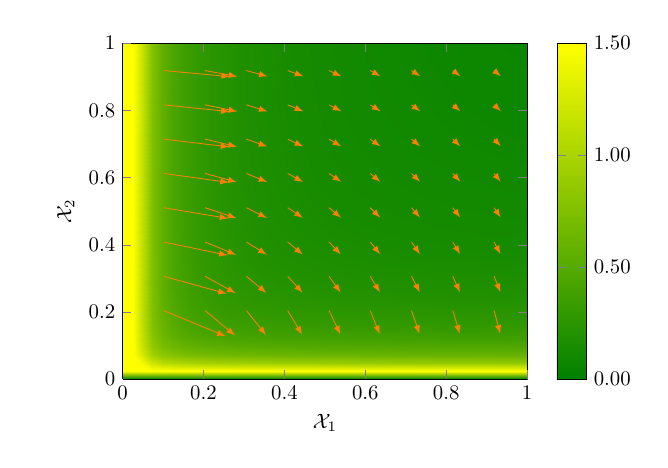}
\caption{Gradient flow of DPO ($\beta=0.1$) with large values truncated at 1.5.}
\label{fig:dpo_grad}
\end{figure}

\subsection{\TheName{}}
\label{sec:g1g2}
The theoretical framework discussed above suggests several directions for improving DPO:

\begin{itemize}
\item \textbf{G1:} Prevent the rejected reward {\(\mathcal{X}_2\)} from rapidly dropping to a very small value, which otherwise halts meaningful updates to the chosen reward { \(\mathcal{X}_1\)} or enhance the gradient update of chosen rewards \(\mathcal{X}_2\).
\item \textbf{G2:} Increase the ratio in Equation \ref{eq:ratio} to allow for more substantial updates to the chosen reward, thereby enhancing the capability of LLMs to generate preferred responses.
\end{itemize}
These adjustments aim to refine the optimization process of RLHF and enhance the performance of LLM in aligning with human preferences.

To achieve the aforementioned goals, we propose incorporating an adjusted preference optimization objective in the loss function of \TheName{} as defined in Formula \ref{eq:pilot}:
\begin{equation}
\label{eq:pilot}
\mathcal{L}_{pilot} := -\frac{1}{2}\mathbb{E}_{(x, y_w, y_l \sim \mathcal{D})}\left[l_{pilot}(\pi_\theta, \pi_\text{pilot})\right],
\end{equation}
where $l_{pilot}(\pi_\theta, \pi_\text{pilot})$ is defined in Formula \ref{eq:pilotDef}.
\begin{align}
\label{eq:pilotDef}
& l_{pilot}(\pi_\theta, \pi_\text{pilot}) :=  \nonumber\\
&  \log \sigma \left(
    \beta \log \frac{\pi_\theta(y_w | x)}{\pi_\text{ref}(y_w |x)}  - 
    \beta \log \frac{\pi_\text{pilot}(\hat{y}_l | x)}{\pi_\text{ref}(\hat{y}_l |x)}
    \right) \nonumber \\
    &+
    \log \sigma \left(
    \beta \log \frac{\pi_\text{pilot}(\hat{y}_w | x)}{\pi_\text{ref}(\hat{y}_w |x)} - 
    \beta \log \frac{\pi_\theta(y_l | x)}{\pi_\text{ref}(y_l |x)}
    \right),
\end{align}
where { $\hat{y}_w$} and {$\hat{y}_l$} denote the sub-sequences of { $y_w$} and { $y_l$}, respectively. See details for the construction of {$\hat{y}_w$} and {$\hat{y}_l$} in Section \ref{sec:subsequence}. In order to simplify the calculations, we introduce
{
$\mathcal{Y}_1 = \frac{\pi_\text{pilot}(\hat{y}_w|x)}{\pi_\text{ref}(\hat{y}_w|x)}$
} and
{
$\mathcal{Y}_2 = \frac{\pi_\text{pilot}(\hat{y}_l|x)}{\pi_\text{ref}(\hat{y}_l|x)}$
}. 
Let us denote the length of the token sequence of $y$ by $T$, with $y_t$ representing the token at the $t$-th index and $y_{<t}$ denoting all tokens preceding the $t$-th index. Given that $\hat{y}$ is a subsequence of $y$, and considering $\pi_\star(y|x)$ can be represented as $\prod_{t=1}^T \pi_\star(y_t|y_{<t},x)$ on a token level basis, where $\star$ belongs to the set \{\text{pilot}, \text{ref}, $\theta$\}, we can express
$\mathcal{X}_2 = p_2 \mathcal{Y}_2$
and
$\mathcal{X}_1 = p_1 \mathcal{Y}_1,$
where $\mathcal{X}_1$ and $\mathcal{X}_2$ are defined in Section \ref{subsec:DPO}. $p_1$ and $p_2$ represent the product of the token probability ratios for the remaining tokens in sequences $\mathcal{X}_1$ and $\mathcal{X}_2$, excluding the sub-sequences $\mathcal{Y}_1$ and $\mathcal{Y}_2$.
  
The $\pi_\text{pilot}$ in Formula \ref{eq:pilotDef} is the guiding policy model to steer the reward updates,
we then demonstrate the advantages of the adjusted preference optimization objective to be leveraged to enhance preference optimization:
\begin{theorem}
\label{thm:pilot_grad}
The partial derivatives of { $l_\text{pilot}$} with respect to {$\mathcal{X}_1$} and { $\mathcal{X}_2$}
are given by:
\begin{small}
\begin{equation}
\label{eq:pilot_pd1}
\frac{\partial l_{\text{pilot}}}{\partial \mathcal{X}_1} =
\frac{
  \beta \mathcal{Y}_2^\beta
}{
  \mathcal{X}_1 (\mathcal{X}_1^\beta + \mathcal{Y}_2^\beta) 
}
\end{equation}
\end{small}

\begin{small}
\begin{equation}
\label{eq:pilot_pd2}
\frac{\partial l_{\text{pilot}}}{\partial \mathcal{X}_2} =
-\frac{
  \beta \mathcal{X}_2^{\beta - 1}
}{
  \mathcal{Y}_1^\beta +   \mathcal{X}_2^\beta
}
\end{equation}
\end{small}
\end{theorem}
\begin{proof}
We defer the detailed proof to Appendix \ref{app:proof}.
\end{proof}

From Theorem \ref{thm:pilot_grad}, we can observe that the gradients of \(\mathcal{X}_1\) and \(\mathcal{X}_2\) depend on \(\mathcal{Y}_1\) and \(\mathcal{Y}_2\), respectively. Consequently, by manipulating \(\mathcal{Y}_1\) and \(\mathcal{Y}_2\), we can control the gradient flow within the alignment method, thereby influencing the updates to the chosen and rejected rewards. We present the visualized representation of these functions in Figure \ref{app:fig:partial-x} in the Appendix.

\begin{theorem}
\label{thm:grad}
The partial derivative {
$\lvert
\frac{\partial l_{\text{pilot}}}{\partial \mathcal{X}_1}
\rvert $
} increases as {$\mathcal{Y}_2$} increases, while
the partial derivative {$
\lvert
\frac{\partial l_{\text{pilot}}}{\partial \mathcal{X}_2}
\rvert
$} decreases as {$\mathcal{Y}_1$} increases. 
\end{theorem}
\begin{proof}
Please see detailed proof in Appendix \ref{app:proof}.
\end{proof}
As preference alignment algorithms enhance the generation probability of preferred text while diminishing that of non-preferred text during fine-tuning, we have $p_2 < 1$. Consequently, \(\mathcal{Y}_2 > \mathcal{X}_2\). 
Comparing Formula \ref{eq:pilot_pd1}  with Formula \ref{eq:pd1}, we can see that the difference lies in just one variable. For instance, $\mathcal{X}_2$ in Formula \ref{eq:pd1} is replaced with $\mathcal{Y}_2$ to derive Formula \ref{eq:pilot_pd1}. 
Based on Theorem \ref{thm:grad}, we then have:
\begin{align}%
\label{eq:adjustedCompare}
\left\lvert
\frac{\partial l_{\text{pilot}}}{\partial \mathcal{X}_1}
\right\rvert  > \left\lvert
\frac{\partial l_{\text{DPO}}}{\partial \mathcal{X}_1}
\right\rvert 
\end{align}%
Formula \ref{eq:adjustedCompare} reveals that \TheName{} can enlarge the gradient of chosen rewards, which enhances the updating of chosen rewards. Consequently, \TheName{} boosts the generation of preferred responses.

\begin{theorem}
\label{thm:ratio}
Let
{ $\pi_\text{pilot} = \pi_\theta$} and
{ $z = \frac{\mathcal{Y}_1}{\mathcal{Y}_2} $},
for each pairwise preference instance
{
$(x, y_w, y_l) \in \mathcal{D}$%
}, 
the ratio between the increase in the probability of a human-preferred response and the decrease in the probability of a human-dispreferred response is given by:
\begin{equation}
  \left \lvert \frac{\partial l_{\text{pilot}}}{\partial \mathcal{X}_1} /
\frac{\partial l_{\text{pilot}}}{\partial \mathcal{X}_2}
\right \rvert = \frac{\mathcal{X}_2}{\mathcal{X}_1} \cdot f(z),
\label{eq:pilot_ratio}
\end{equation}
where
\begin{equation}
f(z) = \frac{1}{p_2^\beta}
\frac{z^\beta  + p_2^\beta}{p_1^\beta z^\beta  + 1}.
\end{equation}
$f(z)$ is a monotonic function of {$z$}.
When { $p_1 p_2 < 1$ }, the function \( f(z) \) is a monotonically increasing function of \( z \).
Conversely,  when { $p_1 p_2 > 1$ }, the function \( f(z) \) is a decreasing function of \( z \).
Furthermore, { $f(z) > 1$ } when { $p_1 p_2 < 1$ }. 
\end{theorem}
\begin{proof}
Please see detailed proof in Appendix \ref{app:proof}.
\end{proof}
When the rejected reward decreases rapidly, it leads to $p_1 p_2 < 1$. Consequently, this results in \( f(z) > 1 \), which boosts the ratio value given by Equation \ref{eq:pilot_ratio}. As \( z = \frac{\mathcal{Y}_1}{\mathcal{Y}_2} \) increases, \( f(z) \) also increases throughout the training process. This behavior aligns with our goal G2, thereby enhancing the capability of LLM to generate preferred text. We present the visual representation of $f(z)$ in Figure \ref{app:fig:fzs} in the Appendix.
A comparison between Figure \ref{fig:dpo_pilot_cmp} (c) and (g) illustrates an example of how Theorem \ref{thm:ratio} works, wherein both DPO and \TheName{} are trained on the same base model using the same preference dataset.

\subsection{Sub-sequence Construction}
\label{sec:subsequence}

Based on Theorems \ref{thm:pilot_grad}, \ref{thm:grad}, and \ref{thm:ratio}, we derive sub-sequences $\hat{y}_w$ and $\hat{y}_l$ from the sequences $y_w$ and $y_l$, respectively. $\hat{y}_w$ and $\hat{y}_l$ serve as indicators to guide the refinement of updates for chosen and rejected reward adjustments. Let $l_1$ and $l_2$ denote the lengths of the sequences $y_w$ and $y_l$. We define $l_c$ as the minimum length between $l_1$ and $l_2$. From the pairs $(y_w, y_l)$, we randomly select preference data pairs $(\hat{y}_w, \hat{y}_l)$ with lengths $(r_1 \cdot l_c, r_2 \cdot l_c)$, where $r_1$ and $r_2$ are hyper-parameters. 

While generating $\hat{y}_w$ and $\hat{y}_l$ for both the \textit{pilot} model and the reference model, we can exploit the same random index or different random indices. Utilizing the same random index ensures that the sub-sequences are constructed from an identical set of tokens. Conversely, employing different random indices results in sub-sequences derived from distinct sets of tokens. We refer to the setting with the same index as \textit{Pilot}$_s$ and that with different indices as \textit{Pilot}$_d$. As different random indices may introduce an element of randomness into the learning space so as to allow \TheName{} to explore more thoroughly and avoid overfitting with superb performance, we exploit \textit{Pilot}$_d$ in \TheName{}.

Adjusting $r_1$ and $r_2$ is critical to the gradient changes associated with the chosen and rejected rewards during the preference optimization process, as indicated by Theorem \ref{thm:grad}. Smaller values of $r_2$ lead to shorter \textit{pilot} sequences, which in turn increases $\mathcal{Y}_2$. As the model converges, the likelihood of encountering tokens from $\hat{y}_l$ decreases. Hence, decreasing $r_2$ typically leads to larger magnitudes of the partial derivatives $\left| \frac{\partial l_{\text{pilot}}}{\partial \mathcal{X}_1} \right|$. Then, the chosen rewards are updated rapidly and the capability of generating human-preferred responses is improved. However, to conserve the semantic meanings of responses, we empirically set $r_1 \geq 0.6$ and $r_2 \geq 0.6$. Moreover, some randomness may exist in the subsequence construction and training process, thus, we fine-tune the values of $r_1$ and $r_2$ to achieve superb performance compared with that of $r_1 = 0.6$ and $r_2 = 0.6$ (see details in Section \ref{subsubsec:pilot}).

\begin{table*}[t]
\begin{small}
\centering
\begin{tabular}{c cccc cccc}
\toprule
\multirow{2}{*}{Methods} &
\multicolumn{4}{c}{Llama-3.1 instruct 8B} &
\multicolumn{4}{c}{Llama-3.1 base 8B}\\
\cmidrule(lr){2-9}
    & Score$_1$    & Score$_2$    & Score$_\text{avg}$  & Token$_\text{len}$
    & Score$_1$    & Score$_2$    & Score$_\text{avg}$  & Token$_\text{len}$ \\
\cmidrule(lr){2-5}
\cmidrule(lr){6-9}
SFT & 8.40 & 7.54 & 7.97 & 287 
    & 7.47 & 6.60 & 7.03 & 247 \\
DPO \citep{rafailov2024direct} & 8.27 & 7.36 & 7.82 & 329
    & 7.29 & 6.41 & 6.85 & 263 \\
NCA \citep{chen2024noise} & 8.13 & 7.21 & 7.67 & 308
    & 7.47 & 6.75 & 7.11 & 256 \\
BCO \citep{Jung2024bco} & 8.22 & 7.25 & 7.74 & 305 
    & 7.44 & 6.43 & 6.94 & 278 \\
IPO \citep{azar24aipo} & \textbf{8.57} & 7.51 & 8.04 & 383
    & 7.51 & \textbf{7.00} & 7.26 & 255 \\
SamPO \citep{Lu2024EliminatingBL} & 8.34 & 7.66 & 8.00 & 289
      & 7.70 & 6.52 & 7.11 & 262 \\
TDPO \citep{zeng2024tokenlevel} & 8.39 & 7.37 & 7.88 & 296
      & 7.30 & 6.38 & 6.84 & 268 \\
\midrule
\TheName{} & 8.38 & \textbf{7.90} & \textbf{8.14} & 312
      & \textbf{7.98} & 6.90 & \textbf{7.44} & 264 \\
\midrule
\multirow{2}{*}{} &
\multicolumn{4}{c}{Qwen-2 instruct 7B} &
\multicolumn{4}{c}{Qwen-2 base 7B}\\
\cmidrule(lr){2-9}
    & Score$_1$    & Score$_2$    & Score$_\text{avg}$  & Token$_\text{len}$
    & Score$_1$    & Score$_2$    & Score$_\text{avg}$  & Token$_\text{len}$ \\
\cmidrule(lr){2-5}
\cmidrule(lr){6-9}
SFT & 8.14 &  7.64&  7.78& 311 
    & 7.94 &  6.80&  7.37& 269    \\
DPO \citep{rafailov2024direct} & 8.44 & 7.99 & 8.21 & 307
    & 7.87 & 6.99 & 7.43 & 293    \\
NCA \citep{chen2024noise} & 8.41 & \textbf{8.12} & 8.27 & 303
    & 7.83 & 7.34 & 7.58 & 291    \\
BCO \citep{Jung2024bco} & 8.49 & 7.97 & 8.23 & 309 
    & 7.87 & 6.62 & 7.25 & 326    \\
IPO \citep{azar24aipo} & 8.31 & 8.04 & 8.17 & 312 
    & 7.65 & \textbf{7.42}  & 7.54 & 351    \\
SamPO \citep{Lu2024EliminatingBL} &8.56&  7.86& 8.21 & 307
      &8.09& 7.09 & 7.59 & 320  \\
TDPO \citep{zeng2024tokenlevel} &8.32& 7.94 & 8.13 & 313
      &7.91& 6.88 & 7.39 & 327  \\
\midrule
\TheName{} & \textbf{8.68}   & 8.04   &    \textbf{8.36}           & 318
      & \textbf{8.26}   & 7.09   &    \textbf{7.67}           & 329 \\
\bottomrule
\end{tabular}
\caption{MT-Bench Results across different model configurations. Here, Score$_1$ refers to the score from the first turn, Score$_2$ to the score from the second turn, and Score$_\text{avg}$ represents the average score. Token$_\text{len}$ indicates the average length of output tokens for each method. we set $r_1=r_2$ for \TheName{} in this experiment.}
\label{tbl:exp}
\end{small}
\end{table*}

\section{Experimental Evaluation}
\label{sec:experiments}
In this section, we compare \TheName{} with 6 state-of-the-art performance optimization algorithms, exploiting 4 model configurations and 8 tasks. We first present the experimental setup. Then, we illustrate the experimental results. Finally, we show the ablation study.

\subsection{Experimental Setup}

We compare \TheName{} against 6 baselines, including DPO \citep{rafailov2024direct},
SamPO \citep{Lu2024EliminatingBL}, IPO \citep{azar24aipo}, Token-Level DPO \citep{zeng2024tokenlevel}, NCA \citep{chen2024noise}, and BCO \citep{Jung2024bco}. These competitive baselines cover a broad range of methods, addressing issues such as eliminating verbosity, avoiding overfitting, ensuring robust alignment, and more. In addition, we consider Llama-3.1 8B \citep{llama3_1blog} and Qwen-2 7B \citep{qwen2} across two configurations: Instruct and Base, which corresponds to 4 model configurations. For the Instruct configuration, we use the instructed model as the Supervised Fine-Tuned (SFT) model, which has already undergone a Supervised Fine-Tuning phase. In contrast, for the Base configuration, we fine-tune the base model with the UltraChat-200k dataset \citep{ding2023enhancing} to create the SFT model, which enhances the base LLM capacity to follow instructions. We leverage the publicly available UltraFeedback dataset \citep{cui2023ultrafeedback} as human preference data. Each entry in the UltraFeedback dataset follows the format $(x, y_w, y_l)$, designed to reflect human values such as helpfulness and honesty.  

We exploit two open-ended generation benchmarks, i.e., MT-Bench \citep{zheng2023judging} and AlpacaEval-2 \citep{alpaca_eval,dubois2024length} (see details in Appendix). For the conditional benchmarks, we evaluate our models on the following 6 tasks: MMLU in a 5-shot setting \citep{Hendrycks2020MeasuringMM}, GSM8K in an 8-shot setting \citep{cobbe2021training}, PiQA in a 3-shot setting \citep{Bisk2020}, TruthfulQA in a 3-shot setting \citep{lin-etal-2022-truthfulqa}, IFEVAL in a 3-shot setting \citep{zhou2023instructionfollowing}, and ARC in a 3-shot setting \citep{Clark2018ThinkYH}. Please see details of the experimental setup in Appendix \ref{app:exp}.

\subsection{Evaluation of \TheName{}}

In this section, we present the experimental comparison of \TheName{} with 7 state-of-the-art optimization algorithms. We first present the experimental results on two open-ended benchmarks, i.e., MT-Bench and AlpacaEval-2. Then, we show the results on conditioned benchmarks. Finally, we present the training rewards of \TheName{} compared with DPO. 

\begin{table*}[t]
\begin{small}
\centering
\begin{tabular}{c ccc ccc c}
\toprule
Method   &   GSM8K   & MMLU     & PiQA     & TruthfuQA   & IFEval  & ARC    & Avg.\\
\midrule
SFT      &  0.5625   & 0.7060   & 0.8096   &  0.5734     &  \textbf{0.4251} & 0.8582 & 0.6558\\
DPO \citep{rafailov2024direct}     &  0.5989   & 0.7065   & \textbf{0.8112}   &  0.5774     &  0.4140 & 0.8628 & 0.6618\\
NCA \citep{chen2024noise}     &  0.5921   & 0.7057   & 0.8079   &  0.5782     &  0.4140 & 0.8607 &  0.6598\\
BCO \citep{Jung2024bco}     &  0.5898   & 0.7065   & 0.8074   &  0.5776     &  \textbf{0.4251} & 0.8620 & 0.6614 \\
IPO \citep{azar24aipo}     &  \textbf{0.6406}   & 0.7039   & 0.7894   &  \textbf{0.5876}     &  0.3974 & 0.8535 & 0.6620\\
SamPO \citep{Lu2024EliminatingBL}   &  \underline{0.6133}   & \underline{0.7067}   & 0.8074   &  \underline{0.5844}     &  0.3993 & \underline{0.8632} & \underline{0.6623} \\
TDPO \citep{zeng2024tokenlevel}     &  0.5951   & 0.7055   & 0.8089   &  0.5763     &  0.3967 & 0.8589 & 0.6569\\
\midrule
\TheName{}    &  0.6111   & \textbf{0.7069}   & \underline{0.8107}   &  0.5806     &  \underline{0.4196} & \textbf{0.8641} & \textbf{0.6655} \\
\bottomrule
\end{tabular}
\caption{
Evaluation results on conditional benchmarks for various approaches, using Qwen-2 instruct 7B as the base model.
}
\label{tbl:cond-res}
\end{small}
\end{table*}

As shown in Table \ref{tbl:exp}, \TheName{} achieves the highest average score compared to other approaches with the MT-Bench benchmark across various base models. Specifically, \TheName{} significantly outperforms DPO (from 1.83\%, to 8.61\%), which highlights the broad applicability of our proposed method across different base models and confirms its effectiveness through high average scores. Moreover, the table reveals that DPO does not invariably enhance the MT-Bench score, which is in line with previous findings \citep{liu2024lddpo}. This result can be attributed to the limitations of DPO as discussed in Section \ref{subsec:DPO}. In addition, compared to the SFT baseline, most alignment methods tend to produce longer response lengths. Notably, the response length of \TheName{} is similar to that of DPO with negligible length bias brought by the \textit{pilot} term, e.g., \TheName{} has a shorter response length on Llama-3.1 instruct 8B, while it has a longer response length on Qwen-2 instruct 7B compared to DPO. Furthermore, the experimental results confirm the capability of \TheName{} to escape saddle points. IPO and SamPO have similar performance while their response lengths differ significantly. Meanwhile, the average scores of IPO and SamPO are lower than that of \TheName{}, which indicates that SamPO and IPO may become trapped in different local optima. In contrast, \TheName{} utilizes a self-guide scheme to avoid getting trapped in a suboptimal policy.
In addition, on AlpacaEval-2 benchmark, our experimental results show that \TheName{} outperforms DPO by 2.51\% on the LC win rate metrics when evaluating the Llama-3.1 instruct 8B model (see details in Appendix).

As shown in Table \ref{tbl:cond-res}, \TheName{} achieves the highest average score (up to 0.0097) compared with 7 competitive baselines based on the conditional benchmarks. In addition, the experimental results demonstrate that all alignment algorithms improve the average score when compared to the SFT baseline. This implies that these alignment algorithms can enhance the capabilities of LLMs to a certain extent. IPO and SamPO achieve higher scores on the GSM8K benchmark, which may suggest that avoiding overfitting and eliminating length bias could improve the reasoning abilities of LLMs. From Tables \ref{tbl:exp} and \ref{tbl:cond-res}, we can also observe that the performance of different algorithms varies between open-ended and conditional benchmarks. Hence, different alignment algorithms correspond to diverse capability aspects of LLMs.

\begin{figure*}[t]
\centering
\begin{tabular}{cc cc cc cc}
  \multicolumn{2}{c}{
  \includegraphics[width=0.22\linewidth]{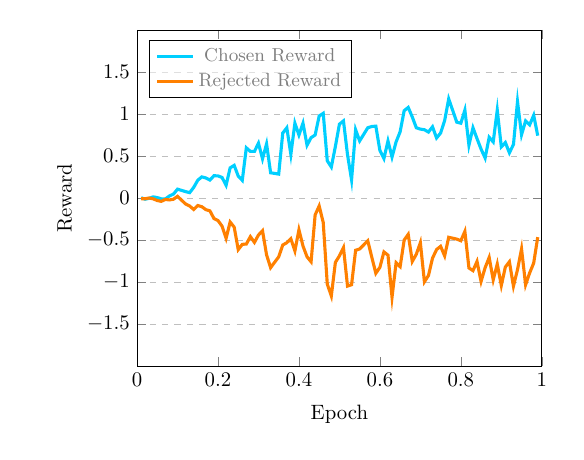}} &
  \multicolumn{2}{c}{
  \includegraphics[width=0.22\linewidth]{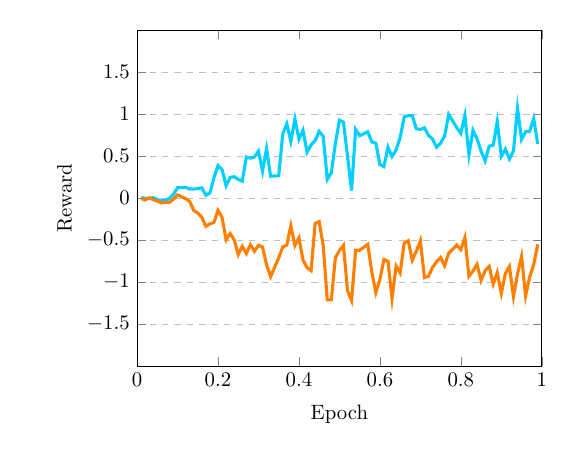}} &
  \multicolumn{2}{c}{
  \includegraphics[width=0.22\linewidth]{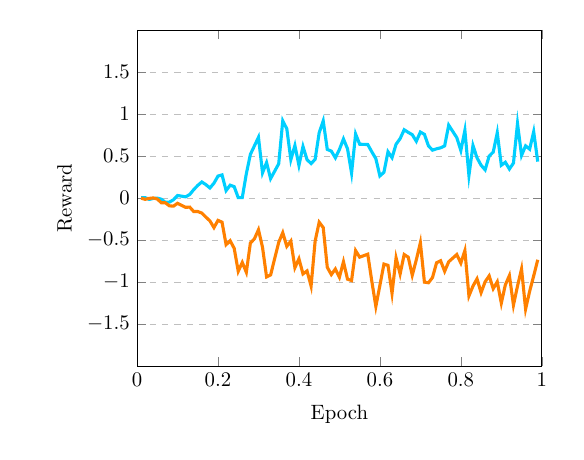}} &
  \multicolumn{2}{c}{
  \includegraphics[width=0.22\linewidth]{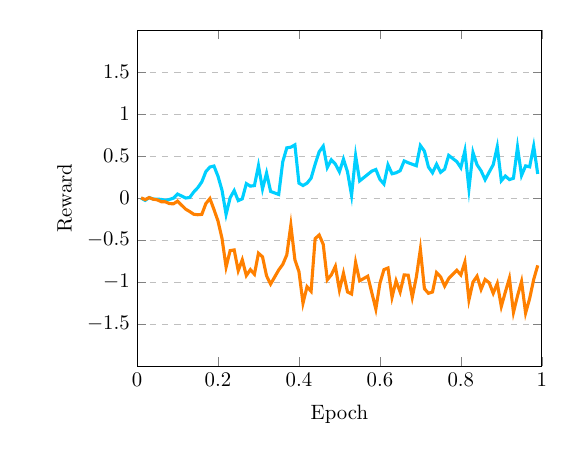}} \\[-5pt]
  \multirow{2}{*}{{\small \hspace{1.8em} (a)}} &\hspace{-2em} {\small CCR = 0.7542} &
  \multirow{2}{*}{{\small \hspace{1.8em} (b)}} &\hspace{-2em} {\small CCR = 0.6816} &
  \multirow{2}{*}{{\small \hspace{1.8em} (c)}} &\hspace{-2em} {\small CCR = 0.5267} &
  \multirow{2}{*}{{\small \hspace{1.8em} (d)}} &\hspace{-2em} {\small CCR = 0.3569} \\
  
                                               &\hspace{-1.8em} {\small CRR = -0.8325} &
                                               &\hspace{-1.8em} {\small CRR = -0.9011} &
                                               &\hspace{-1.8em} {\small CRR = -1.0337}&
                                               &\hspace{-1.8em} {\small CRR = -1.0842} \\ [-5pt]
  
  \end{tabular}
\caption{
Training reward curves for the Llama-3.1 base 8B model using the \TheName{} method:
(a) $r_1 = 0.6$ and $r_2 = 0.6$.
(b) $r_1 = 0.7$ and $r_2 = 0.7$.
(c) $r_1 = 0.8$ and $r_2 = 0.8$.
(d) $r_1 = 0.9$ and $r_2 = 0.9$.
``CCR'' represents the average (last 80 iterations) convergence chosen reward and ``CRR'' represents the average (last 80 iterations) convergence reject reward.
}
\label{-6mm}
\label{fig:curve_vary_r}
\end{figure*}

In order to show the robust performance of \TheName{}, we present the training awards with diverse model configurations in Figure \ref{fig:dpo_pilot_cmp}. In this experimentation, the training reward curve for \TheName{} was generated using hyper-parameters $r_1$ and $r_2$, both set to 0.6. The figure demonstrates that \TheName{} is much more stable than DPO across all the base model configurations. We empirically observe that the patterns of the training rewards for DPO vary significantly across different base models. For instance, the chosen reward of DPO on Llama-3.1 instruct 8B first increases, then drops to a low value, while the chosen reward of DPO on Qwen-2 instruct 7B shows the well-observed decreasing-likelihood phenomenon. In contrast, \TheName{} exhibits consistent reward patterns. These findings reveal that \TheName{} offers greater resilience compared to the DPO method.

\begin{figure}[t]
\includegraphics[width=\linewidth]{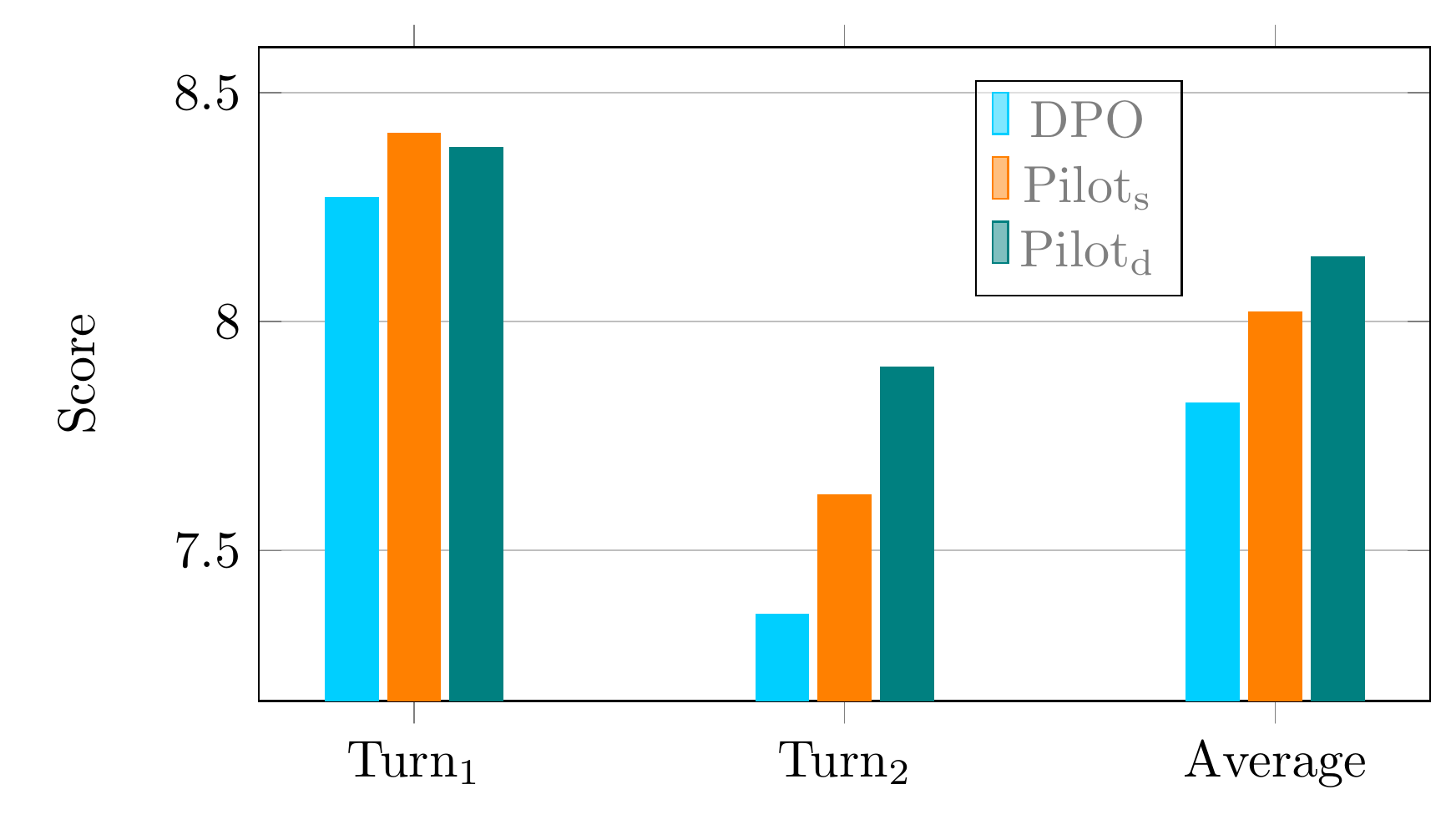}
\caption{MT-Bench Results across different model configurations, using Llama-3.1 instruct 8B as the base model.}
\label{fig:seq_cmp}
\end{figure}

\subsection{Ablation Study}

In this section, we first present experimental results for selecting between \textit{Pilot}$_s$ and \textit{Pilot}$_d$. Next, we analyze the impact of the hyper-parameters $r_1$ and $r_2$ on \textit{Pilot} and overall model performance. We further investigate the behavior of \TheName{} under settings where $r_1 \neq r_2$. Finally, we compare \TheName{} with ORPO, a recent method for preference optimization.

\subsubsection{Sub-sequence Construction}

As shown in Figure \ref{fig:seq_cmp}, we carry out an experiment for the comparison between \textit{Pilot}$_s$ and \textit{Pilot}$_d$ with Llama-3.1 instruct 8B and MT-Bench. The experimental results demonstrate that both \textit{Pilot}$_s$ and \textit{Pilot}$_d$ achieve higher average scores (from 2.56\% to 4.09\%) compared to DPO. This indicates the effectiveness of \TheName{}. Moreover, \textit{Pilot}$_d$ attains a higher average score (1.50\%) than \textit{Pilot}$_s$. This is expected as different random indices introduce an element of randomness into the learning space corresponding to superior performance as explained in Section \ref{sec:subsequence}.

\subsubsection{The \textit{Pilot} Term}
\label{subsubsec:pilot}

While $r_1$ and $r_2$ are critical to \textit{pilot}, we conduct experimentation to evaluate the influence of $r_1$ and $r_2$ on the performance of \TheName{}, including the robustness and the reward patterns.

\textbf{Robustness.} While \textit{pilot} exploits $r_1$ and $r_2$ to regulate the lengths of the token sequences, we carry out an experimentation with diverse $r_1$ and $r_2$ ($r_1 = r_2 = r$) so as to verify the corresponding performance and robustness of \TheName{}. As shown in Figure \ref{fig:score_r}, \TheName{} significantly outperforms DPO in all settings, achieving notably higher scores in the first turn (from 4.25\% to 9.19\%), second turn (from 1.56\% to 10.14\%), and on average (from 3.06\% to 8.61\%). \TheName{} achieves its lowest value at \( r=1 \), yet it still exhibits a relative improvement of 3.06\% over DPO. The length of the generated response tokens remains comparable to that of DPO. These experimental results show the robustness of \TheName{} across varying $r_1$ and $r_2$.

\begin{figure}[t]
\includegraphics[width=\linewidth]{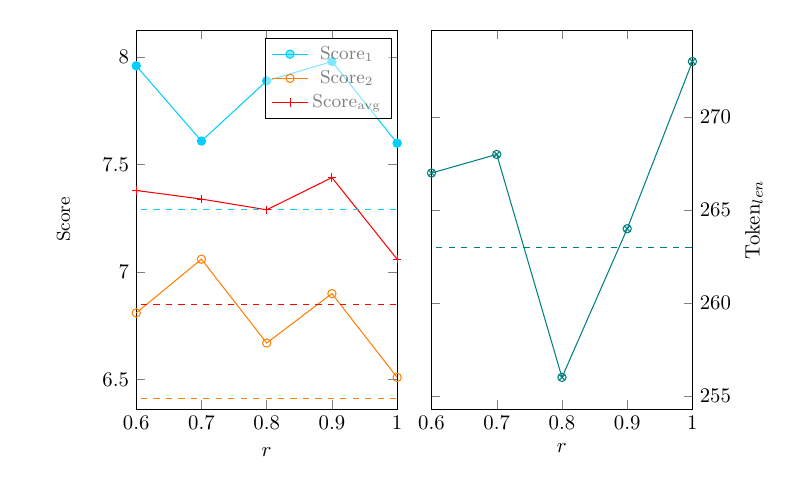}
\caption{
MT-Bench results of \TheName{} across various configurations, using Llama-3.1 base 8B as the base model. The dashed lines represent the score and the token length of DPO.
}
\label{fig:score_r}
\end{figure}

\textbf{Reward Patterns.} As shown in Figure \ref{fig:curve_vary_r}, when $r_1$ and $r_2$ range from 0.6 to 0.9, both the average Convergence Chosen Reward (CCR) and the average Convergence Reject Reward (CRR) over the last 80 iterations of fine-tuning decrease. A more significant CRR corresponds to a modest decrease in rejected rewards, which is in line with G1 explained in Section \ref{sec:g1g2}. Correspondingly, as shown in Figure \ref{fig:score_r}, the average score is negatively correlated with $r$ (from 0.6 to 1.0), with the exception of a fluctuation occurring at $r = 0.9$. This fluctuation may be due to the randomness in sub-sequence construction. As a consequence, in our experimentation, we take the best values of $r$ to achieve excellent performance (see experimental setting details of $r$ in Appendix).

\begin{table}[t]
\begin{small}
\begin{center}
\begin{tabular}{cc ccc}
\toprule
  $r_1$ & $r_2$ & Score$_1$ & Score$_2$ & Score$_\text{avg}$ \\
\midrule
0.9 & 0.5  & 7.72 &6.50  &7.11 \\
0.9 & 0.6  & \textbf{8.01} & \textbf{6.95} & \textbf{7.48} \\
0.9 & 0.7  & 7.84 & 6.89 & 7.37 \\
0.9 & 0.8  & 7.84 & 6.68 & 7.26 \\
0.9 & 0.9  & 7.98 & 6.90 & 7.44 \\
\bottomrule
\end{tabular}
\end{center}
\end{small}
\caption{MT-Bench Results across different $r_1$ and $r_2$ on Llama-3.1 base 8B model.}
\label{tbl:diff-r1-r2}
\end{table}

\textbf{Different $r_1$ and $r_2$.} In previous experiments, we set $r_1 = r_2$ to evaluate model performance. To further investigate the effectiveness of \TheName{}, we conduct additional experiments with different values of $r_1$ and $r_2$. As shown in Table~\ref{tbl:diff-r1-r2}, varying these parameters leads to further improvements in performance. Specifically, setting $r_1 = 0.9$ and $r_2 = 0.6$ achieves an average score of 7.48, outperforming the baseline configuration ($r_1 = r_2 = 0.9$, score = 7.44). This result also represents a significant improvement over DPO, with a relative gain of 9.19\%.

\begin{table}[t]
\begin{small}
\begin{center}
\begin{tabular}{cc cc}
\toprule
  Method & Score$_1$ & Score$_2$ & Score$_\text{avg}$ \\
  \midrule
 DPO         &   7.29           &    \underline{6.41}           & \underline{6.85}        \\ 
 ORPO	     &   \underline{7.40}           &    6.04           &  6.78       \\
 \TheName{}  &   \textbf{8.01}  &    \textbf{6.95}  &  \textbf{7.48} \\
\bottomrule
\end{tabular}
\end{center}
\end{small}
\caption{MT-Bench Results across different methods on Llama-3.1 base 8B model.}
\label{tbl:cmp-orpo}
\end{table}

\textbf{Compared with ORPO.}
ORPO \citep{hong-etal-2024-orpo} presents a novel approach to preference optimization by proposing a unified odds ratio-based framework that does not rely on a separate reference model. This innovative method effectively integrates preference learning into a single training stage, thereby removing the need for an additional alignment step and significantly streamlining the overall optimization process.

We conduct an ablation study to compare \TheName{}, DPO, and ORPO. As shown in Table~\ref{tbl:cmp-orpo}, \TheName{} outperforms both DPO and ORPO by a large margin, which demonstrates the effectiveness of \TheName{}.

\section{Conclusions}
In this paper, we present a novel self-guided direct preference optimization algorithm, i.e., \TheName{}, for aligning LLMs with human preferences. \TheName{} incorporates a \textit{pilot} term in the objective function in order to guide the gradient updates of the rewards during training. We provide a detailed theoretical explanation of \TheName{}. Furthermore, extensive experimental results across various model settings and benchmarks demonstrate the significant advantages (up to 9.19\% higher score) of \TheName{}.

\cleardoublepage
\section*{Limitations}

\TheName{} includes the resampling of a sub-sequence from the logits of the output layer, which introduces extra computational steps. Nevertheless, as demonstrated in Table \ref{app:tbl:time-cost} within the Appendix, this results in a minor increase (up to 0.4\%) in computational overhead. 

While \TheName{} exploits public centralized preference datasets to fine-tune models in order to align LLMs with human values, the datasets may contain unhelpful or misleading preferred information leading to unexpected responses. \TheName{} may be subject to this potential drawback. In addition, the datasets may be distributed in diverse data centers or edge devices \cite{chen2025trustworthy,liu2024enhancing,liu2022distributed,liu2015survey}, which may restrict the application of \TheName{}. In the future, we plan to investigate the adaptation of \TheName{} into a broader setting, e.g., federated learning \cite{Liu2024Fisher,liu2024efficient,jia2024efficient,liu2024aedfl,liu2024fedasmu,liu2024fedasmu,Che2023Federated,liu2023distributed,liu2022multi,Zhang2022FedDUAP,zhou2022efficient} and distributed machine learning \cite{liu2023heterps}.

\bibliography{custom}
    \cleardoublepage
    \appendix
    
    \setcounter{theorem}{0}
    \setcounter{equation}{0}
    \section{Appendix}
    \subsection{Proof of \ref{thm:pilot_grad}, \ref{thm:grad}, and \ref{thm:ratio}}
    \label{app:proof}
    
    \begin{theorem}
    \label{app:thm:pilot_grad}
    The partial derivatives of { $l_\text{pilot}$} with respect to {$\mathcal{X}_1$} and { $\mathcal{X}_2$}
    are given by:
    \begin{small}
    \begin{equation}
    \frac{\partial l_{\text{pilot}}}{\partial \mathcal{X}_1} =
    \frac{
      \beta \mathcal{Y}_2^\beta
    }{
      \mathcal{X}_1 (\mathcal{X}_1^\beta + \mathcal{Y}_2^\beta) 
    }
    \end{equation}
    \end{small}
    
    \begin{small}
    \begin{equation}
    \frac{\partial l_{\text{pilot}}}{\partial \mathcal{X}_2} =
    -\frac{
      \beta \mathcal{X}_2^{\beta - 1}
    }{
      \mathcal{Y}_1^\beta +   \mathcal{X}_2^\beta
    }
    \end{equation}
    \end{small}
    \end{theorem}
    
    \begin{proof}
    By variable substitution, we have:
    \begin{small}
    \begin{align}
      l_{pilot}(\pi_\theta, \pi_\text{ref}) 
        &= \log \left(
        \frac{
          \mathcal{X}_1^\beta
        }{
          \mathcal{X}_1^\beta + \mathcal{Y}_2^\beta
        }
        \right) 
        + \log \left(
        \frac{
          \mathcal{Y}_1^\beta
        }{
          \mathcal{Y}_1^\beta + \mathcal{X}_2^\beta
        }
        \right)
    \end{align}
    \end{small}
    For {$\frac{\partial l_{\text{pilot}}}{\partial \mathcal{X}_1}$},
    \begin{small}
    \begin{align}
    \frac{\partial l_{\text{pilot}}}{\partial \mathcal{X}_1} & = 
        \frac{
          \mathcal{X}_1^\beta + \mathcal{Y}_2^\beta
        }{
          \mathcal{X}_1^\beta
        } \left(
        \frac{\beta \mathcal{X}_1^{\beta-1} }{\mathcal{X}_1^\beta + \mathcal{Y}_2^\beta} -
        \frac{\beta \mathcal{X}_1^{2\beta-1} }{(\mathcal{X}_1^\beta + \mathcal{Y}_2^\beta)^2}
        \right) \nonumber \\
        & = \frac{
      \beta \mathcal{Y}_2^\beta
    }{
      \mathcal{X}_1 (\mathcal{X}_1^\beta + \mathcal{Y}_2^\beta) 
    } 
    \end{align}
    \end{small}
    For {$\frac{\partial l_{\text{pilot}}}{\partial \mathcal{X}_2}$},
    \begin{small}
    \begin{align}
    \frac{\partial l_{\text{pilot}}}{\partial \mathcal{X}_2} & =
    \frac{\mathcal{Y}_1^\beta + \mathcal{X}_2^\beta}{\mathcal{Y}_1^\beta}
    \frac{-\mathcal{Y}_1^\beta \beta \mathcal{X}_2^{\beta-1}}{(\mathcal{Y}_1^\beta + \mathcal{X}_2^\beta)^2} \nonumber\\
    & =  -\frac{
      \beta \mathcal{X}_2^{\beta - 1}
    }{
      \mathcal{Y}_1^\beta +   \mathcal{X}_2^\beta
    }
    \end{align}
    \end{small}
    \end{proof}
    
    \begin{theorem}
    \label{app:thm:grad}
    The partial derivative {
    $\lvert
    \frac{\partial l_{\text{pilot}}}{\partial \mathcal{X}_1}
    \rvert $
    } increases as {$\mathcal{Y}_2$} increases, while
    the partial derivative {$
    \lvert
    \frac{\partial l_{\text{pilot}}}{\partial \mathcal{X}_2}
    \rvert
    $} descreases as {$\mathcal{Y}_1$} increases.
    \end{theorem}
    
    \begin{proof}
    For {
    $\lvert
    \frac{\partial l_{\text{pilot}}}{\partial \mathcal{X}_1}
    \rvert $
    }, we have
    \begin{small}
    \begin{align}
    \frac{\partial
    \lvert\frac{\partial l_{\text{pilot}}}{\partial \mathcal{X}_1}\rvert
    }{\partial \mathcal{Y}_2
    }  & = \frac{
    \beta^2 \mathcal{Y}^{\beta-1} \mathcal{X}_1^\beta
    }{\mathcal{X}_1 (\mathcal{Y}_2^\beta + \mathcal{X}_1^\beta)^2} \\
      & > 0 
    \end{align}
    \end{small}
    
    For {
    $\lvert
    \frac{\partial l_{\text{pilot}}}{\partial \mathcal{X}_2}
    \rvert $
    }, we have
    \begin{small}
    \begin{align}
    \frac{\partial
    \lvert\frac{\partial l_{\text{pilot}}}{\partial \mathcal{X}_2}\rvert
    }{\partial \mathcal{Y}_1
    }  & = -\frac{
    \beta^2 \mathcal{X}_2 ^ {\beta-1} \mathcal{Y}_1^{\beta-1}
    }{(\mathcal{Y}_1^\beta + \mathcal{X}_2^\beta)^2} \\
      & < 0 
    \end{align}
    \end{small}
    \end{proof}
    
    \begin{theorem}
    \label{app:thm:ratio}
    Let
    { $\pi_\text{pilot} = \pi_\theta$} and
    { $z = \frac{\mathcal{Y}_1}{\mathcal{Y}_2} $},
    for each pairwise preference instance
    {
    $(x, y_w, y_l) \in \mathcal{D}$
    }, 
    the ratio of the increase in the probability of a human-preferred response to the decrease in the probability of a human-dispreferred response is given by:
    \begin{small}
    \begin{equation}
      \left \lvert \frac{\partial l_{\text{pilot}}}{\partial \mathcal{X}_1} /
    \frac{\partial l_{\text{pilot}}}{\partial \mathcal{X}_2}
    \right \rvert = \frac{\mathcal{X}_2}{\mathcal{X}_1} \cdot f(z),
    \label{app:eq:pilot_ratio}
    \end{equation}
    \end{small}
    where
    \begin{small}
    \begin{equation}
    f(z) = \frac{1}{p_2^\beta}
    \frac{z^\beta  + p_2^\beta}{p_1^\beta z^\beta  + 1}
    \end{equation}
    \end{small}
    is a monotonic function of
    {$z$}.
    When { $p_1 p_2 < 1$ },
    the function { $f(z)$ } is increasing.
    Conversely,  if { $p_1 p_2 > 1$ }, the function {$f(z)$ } is decreasing.
    Furthermore, { $f(z) > 1$ } if { $p_1 p_2 < 1$ }.
    \end{theorem}
    
    \begin{proof}
    \begin{small}
    \begin{align}
      \left \lvert \frac{\partial \mathcal{L}_{\text{pilot}}}{\partial \mathcal{X}_1} /
    \frac{\partial \mathcal{L}_{\text{pilot}}}{\partial \mathcal{X}_2}
    \right \rvert  & =
    \frac{\mathcal{Y}_2}{\mathcal{X}_1}
    \frac{\mathcal{Y}_2^{\beta-1}}{\mathcal{X}_2^{\beta-1}}
    \frac{\mathcal{Y}_1^\beta  + \mathcal{X}_2^\beta}{\mathcal{X}_1^\beta  + \mathcal{Y}_2^\beta} \\
    & = \frac{\mathcal{X}_2}{\mathcal{X}_1}
    \frac{\mathcal{Y}_2^\beta}{\mathcal{X}_2^\beta}
    \frac{\mathcal{Y}_1^\beta  + \mathcal{X}_2^\beta}{\mathcal{X}_1^\beta  + \mathcal{Y}_2^\beta} 
    \end{align}
    \end{small}
    Let {$\mathcal{X}_2 = p_2 \mathcal{Y}_2, \mathcal{X}_1 = p_1 \mathcal{Y}_1$}, and 
    {$z = \frac{\mathcal{Y}_1}{\mathcal{Y}_2}$}, we then have
    \begin{small}
    \begin{align}
      f(z) & = 
    \frac{\mathcal{Y}_2^\beta}{\mathcal{X}_2^\beta}
    \frac{\mathcal{Y}_1^\beta  + \mathcal{X}_2^\beta}{\mathcal{X}_1^\beta  + \mathcal{Y}_2^\beta} \nonumber \\
    & =
    \frac{\mathcal{Y}_2^\beta}{(p_2\mathcal{Y}_2)^\beta}
    \frac{\mathcal{Y}_1^\beta  + (p_2\mathcal{Y}_2)^\beta}{(p_1\mathcal{Y}_1)^\beta  + \mathcal{Y}_2^\beta} \nonumber \\
     & = \frac{1}{p_2^\beta}
    \frac{z^\beta  + p_2^\beta}{p_1^\beta z^\beta  + 1}
    \end{align}
    \end{small}
    
    The derivative of {$f(z)$} with respect to {$z$} is
    \begin{small}
    \begin{align}
      \frac{\partial f(z)}{\partial z} &\propto \beta z ^ {\beta - 1} (p_1^\beta z^\beta  + 1) -
      (z^\beta  + p_2^\beta) p_1^\beta \beta z^{\beta-1}  \nonumber \\
      &= \beta z ^ {\beta - 1} - \beta (p_1 p_2)^\beta z ^ {\beta-1} \nonumber\\
      &= \beta\left(1-(p_1p_2)^\beta\right) z ^ {\beta-1}
    \end{align}
    \end{small}
    Since {$z = \frac{\mathcal{Y}_1}{\mathcal{Y}_2} > 0$},
    whether {$\frac{\partial f(z)}{\partial z} > 0 $} or {$\frac{\partial f(z)}{\partial z} <  0 $} is contingent on the value of {$p_1 p_2$}.
    Therefore, if {$p_1 p_2 < 1$},
    the function {$f(z)$} is increasing.
    Conversely,  if {$p_1 p_2 > 1$}, the function {$f(z)$} is decreasing.
    \end{proof}

    \subsection{Experimental Setup}
    \label{app:exp}
    
    To ensure a fair comparison among different methods, we employ the same general settings for all baselines, which are detailed in Table \ref{tbl:settings}. Additionally, we set $\beta=0.1$ for all baselines. For the proposed \TheName{} method, we set $r_1 = r_2$ by default and performed a grid search over the range \{0.6, 0.7, $\cdots$, 1.0\}. Table \ref{app:tbl:r1r2} shows the parameters we select.
    We carry out our experiments on 4 A800-80G GPUs.

    \begin{table}[t]
    \begin{tabular}{cc}
    \toprule
    Model  &  $r_1, r_2$ \\
    \midrule
    Llama-3.1 instruct 8B &  $r_1=1.0$, $r_2=1.0$ \\
    Llama-3.1 Base  8B    &  $r_1=0.9$, $r_2=0.9$  \\
    Qwen-2 instruct 7B    &  $r_1=0.9$, $r_2=0.9$   \\
    Qwen-2 base 7B        &  $r_1=0.6$, $r_2=0.6$  \\
    \bottomrule
    \end{tabular}
    \caption{The hyper-parameters we used for \TheName{} in the experiments reported in Table \ref{tbl:exp}}
    \label{app:tbl:r1r2}
    \end{table}
    
    \begin{table}[t]
    \centering
    \begin{tabular}{cc cc cc}
    \toprule
    Phase & LR   & BS  & Epoch  &  LS       & WP  \\
    \midrule
    SFT   & 2e-5 & 128 &   3    &  cosine   & 0.1 \\
    PO    & 5e-7 & 128 &   1    &  cosine   & 0.1 \\
    \bottomrule
    \end{tabular}
    \caption{
    The general training settings for the Supervised Fine - Tuning (SFT) phase and Preference Optimization (PO) phase include Learning Rate (LR), Batch Size (BS), Epoch, Learning Rate Schedule (LS), and Warmup Phase (WP).}
    \label{tbl:settings}
    \end{table}
    
    As a large-scale, finely detailed, and diverse dataset, UltraFeedback dataset \citep{cui2023ultrafeedback} comprises approximately 64,000 prompts sourced from a wide array of origins. MT-Bench consists of a multi-turn question set with 80 questions designed to evaluate the capabilities of a model in multi-turn conversation and instruction-following. In our experimentation, we utilize a single-answer grading mode, where GPT-4 \citep{openai23_gpt4} assigns a score out of 10 for each turn. We report the average score per turn across our experiments.
    

    
    \begin{figure}[t]
    \includegraphics[width=\linewidth]{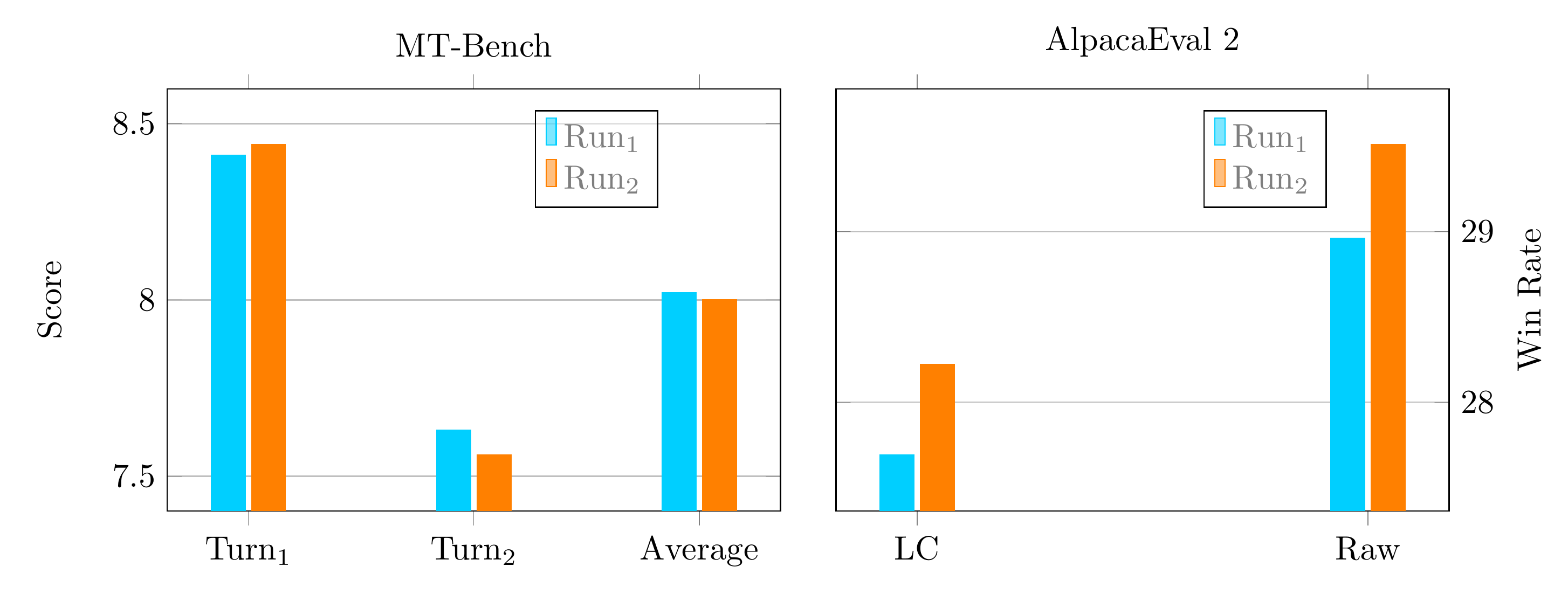} 
    \caption{Performance Metrics of Various Runs on MT-Bench and AlpacaEval-2.}
    \label{app:fig:bench_var}
    \end{figure}
    
    \begin{figure*}[t]
    \begin{center}
    \begin{tabular}{ccc}
      \includegraphics[width=0.29\linewidth]{figures/dpo_reward.pdf} & 
      \includegraphics[width=0.29\linewidth]{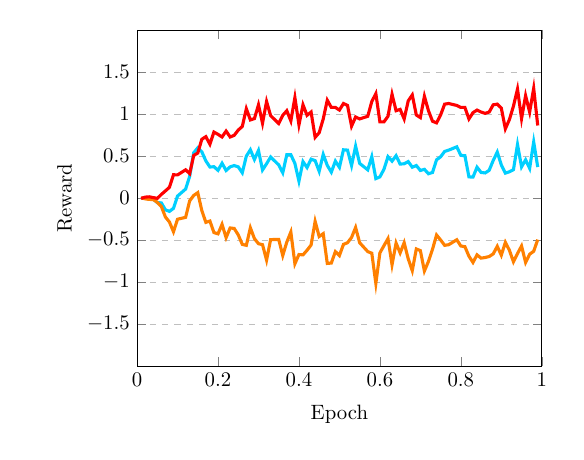} &
      \includegraphics[width=0.29\linewidth]{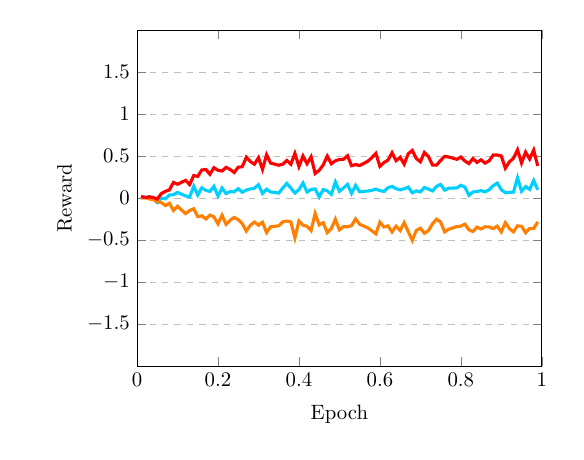}  \\
     {\qquad (a) DPO} & {\qquad (b)  \TheName{}} & {\qquad (c) NCA }\\ 
      \includegraphics[width=0.29\linewidth]{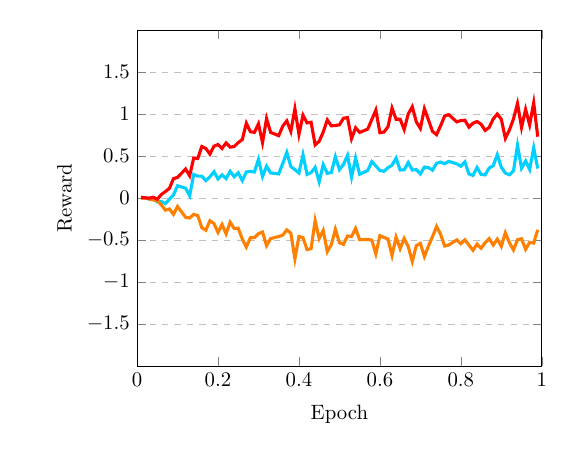} & 
      \includegraphics[width=0.29\linewidth]{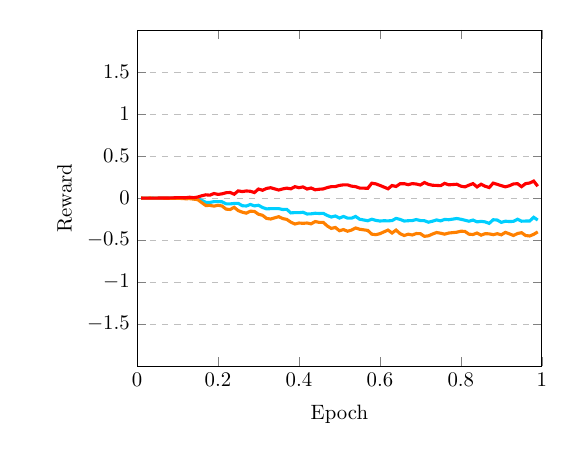} &
      \includegraphics[width=0.29\linewidth]{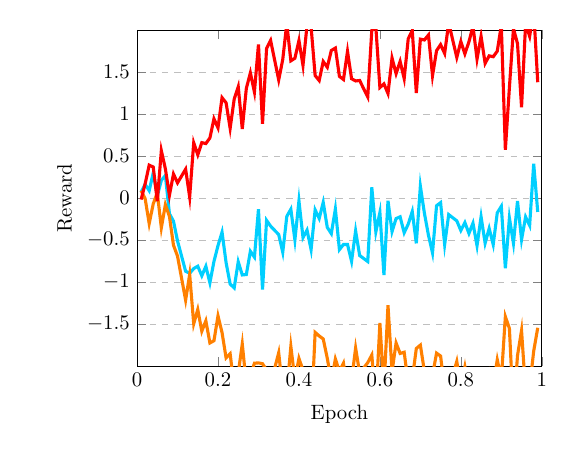} \\
    {\qquad  (d) BCO } & {\qquad (e) IPO } & {\qquad (f) SamPO}\\
    \end{tabular}
    \end{center}
    \caption{Training reward curves for the Llama-3.1 instruct 8B model using various alignment methods.}
    \label{app:fig:rewards}
    \end{figure*}

    \begin{figure*}[t]
    \begin{center}
    \begin{tabular}{ccc}
      \includegraphics[width=0.29\linewidth]{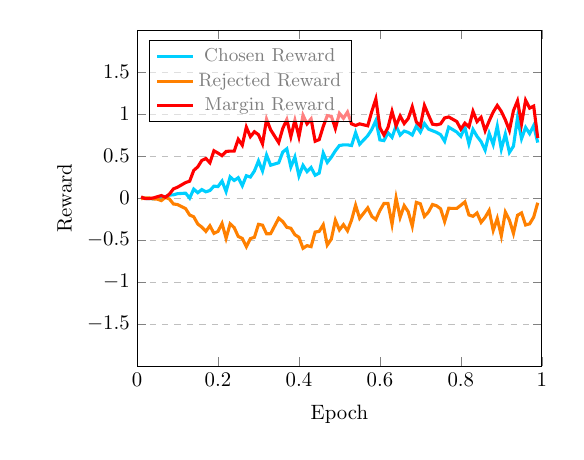} & 
      \includegraphics[width=0.29\linewidth]{figures/reward_pv4.pdf} &
      \includegraphics[width=0.29\linewidth]{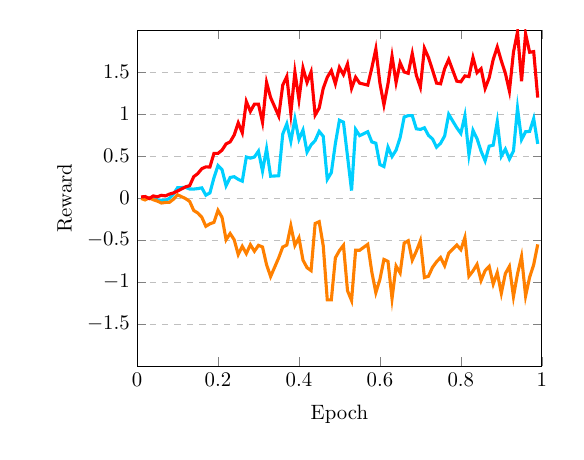}  \\
      {\qquad (a)} & {\qquad (b)} &  {\qquad (c)} \\ 
      \includegraphics[width=0.29\linewidth]{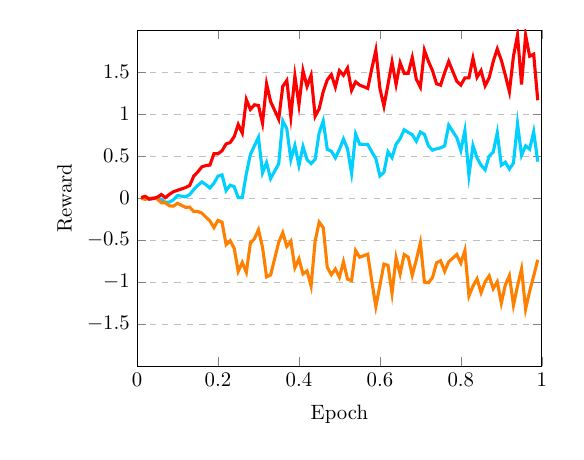} & 
      \includegraphics[width=0.29\linewidth]{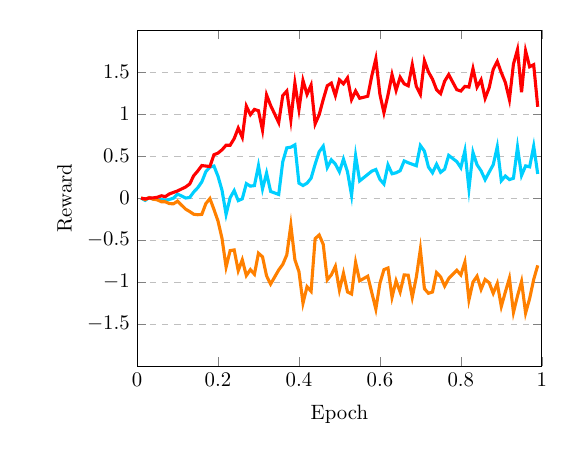} &
      \includegraphics[width=0.29\linewidth]{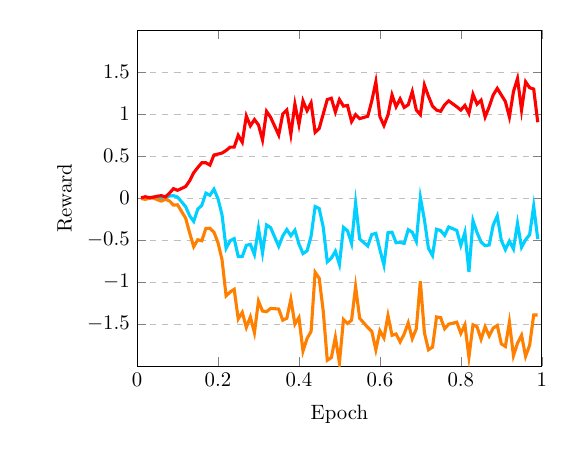}  \\ 
      {\qquad (d)} & {\qquad (e)} & {\qquad (f)}  \\
    \end{tabular}
    \end{center}
    \caption{
    Training reward curves for the Llama-3.1 base 8B model using the DPO and \TheName{} methods:
    (a) DPO.
    (b) \TheName{} with $r_1 = 0.6$ and $r_2 = 0.6$.
    (c) \TheName{} with $r_1 = 0.7$ and $r_2 = 0.7$.
    (d) \TheName{} with $r_1 = 0.8$ and $r_2 = 0.8$.
    (e) \TheName{} with $r_1 = 0.9$ and $r_2 = 0.9$.
    (f) \TheName{} with $r_1 = 1.0$ and $r_2 = 1.0$.
    }
    \label{app:fig:rewards_pv_r}
    \end{figure*}
    
\section{Complexity}

\TheName{} entails a novel technique where we resample subsequences from the probability distributions (logits) generated by the output layer. This process introduces supplementary computational stages into the workflow. Despite this added complexity, as detailed in Table \ref{app:tbl:time-cost}, the resultant increase in computational overhead remains modest (up to 0.4\%) additional computational time.
    
\section{More Experiments}

We also employ the AlpacaEval-2 \citep{alpaca_eval,dubois2024length} benchmark for evaluation. AlpacaEval-2 operates on a fixed set of 805 instructions, for which both the base model and the evaluated model generate responses. A GPT-based model then compares these responses to determine the win rate. In our experiments, we report both the length-controlled win rate and the raw win rate. We utilize the \textit{weighted\_alpaca\_eval\_gpt4\_turbo} configuration recommended by the AlpacaEval-2 library \citep{dubois2024length} for this evaluation. We report the results in Table \ref{app:tbl:more-exp}.
    
From the experimental results, we can observe that \TheName{} significantly outperforms the baselines in the LC win rate metric (up to 5.14\%) with Llama-3.1 instruct 8B. However, unlike the experiments on MT-Bench, \TheName{} does not surpass the baselines on Qwen-2 instruct 7B model. This indicates the effectiveness of alignment optimization might be benchmark-dependent. Conducting a rigorous evaluation of large language models remains a research direction of significant importance.
    
\begin{table*}[t]
\begin{small}
\centering
\begin{tabular}{c ccc ccc}
\toprule
\multirow{2}{*}{Methods} &
\multicolumn{3}{c}{Llama-3.1 instruct 8B} &
\multicolumn{3}{c}{Qwen-2 instruct 7B}\\
    \cmidrule(lr){2-7}
        & LC win rate    &  Raw win rate     & Token$_\text{len}$
        & LC win rate    &  Raw win rate     & Token$_\text{len}$ \\
    \cmidrule(lr){2-4}
    \cmidrule(lr){5-7}
    SFT  & 26.84 & 27.77 & 459 
         & 20.98 & 22.20 & 418 \\
    DPO \citep{rafailov2024direct} & 27.53 & 28.35 & 438
         & 24.26 & 24.50 & 414\\
    NCA \citep{chen2024noise}  & 26.33 & 27.77 & 441 
         & 21.94 & 21.75 & 409\\     
    BCO \citep{Jung2024bco} & 28.03 & \textbf{29.32} & 435
         & 23.76 & 23.95 & 411\\     
    IPO \citep{azar24aipo} & 27.07 & 25.82 & 459
         & \textbf{29.03} & 25.68 & 411\\     
    SamPO \citep{Lu2024EliminatingBL} & 27.45 & 27.69 & 443 
         & 24.57 & \textbf{26.60} & 426\\
    \midrule
    \TheName{} & \textbf{28.22} & 28.96 & 444 
         & 23.89 & 24.86 & 419\\          
    \bottomrule
    \end{tabular}
      \caption{AlpacaEval-2 Results across different model configurations. Token$_\text{len}$ indicates the average length of output tokens for each method.}
      \label{app:tbl:more-exp}
    \end{small}
    \end{table*}
    
    \begin{table}[t]
    \centering
    \begin{tabular}{cc}
    \toprule
    Method & Training Time \\
    \midrule
    DPO    & 6h22m22s \\
    \TheName{}  & 6h24m07s \\
    \bottomrule
    \end{tabular}
    \caption{
    Training time cost of DPO and \TheName{}.
    }
    \label{app:tbl:time-cost}
    \end{table}
    
\section{Training Reward Curves}
    
As discussed in Section \ref{sec:subsequence}, adjusting the values of $r_1$ and $r_2$ can affect the optimization process, resulting in different reward curve shapes. In Figure \ref{app:fig:rewards_pv_r}, we present the full training curves for \TheName{} and DPO. The results show that setting $r_1$ and $r_2$ to smaller values can lead to an increase in the magnitude of the reward values at the end of the fine-tuning stage. We also present the training reward curves of the baselines in Figure \ref{app:fig:rewards}.
    
\section{Variance}

In this paper, we carry out extensive experiments using both the MT-Bench and AlpacaEval-2 frameworks. Both MT-Bench and AlpacaEval-2 utilize GPT for evaluating responses, we investigate whether there are significant discrepancies in the assessments of GPT with identical content across different calls. To explore this, we conducted a test by querying GPT twice with the same response content and present our findings in Figure \ref{app:fig:bench_var}. The experimental results indicate that while MT-Bench yields relatively consistent outcomes with lower variance, AlpacaEval-2 demonstrates a notably higher variance under similar conditions.

\section{Future Work}
As discussed in the Limitations section, \TheName{} introduces additional computational steps. To address this, we aim to design a novel architecture for \TheName{} that reduces the associated computational overhead.  
We also plan to evaluate our method in long-context scenarios \citep{liu2025comprehensive,zhu2024psc} and recommendation systems \citep{zhu2025csdm}, as recommendations are inherently driven by user preferences.

Additionally, we intend to explore the applicability of \TheName{} in broader settings, such as learning with non-Independent and Identically Distributed (non-IID) data under federated learning frameworks. We also plan to investigate the use of diverse models or enhanced architectures within the policy framework—specifically, the \textit{pilot} model in \TheName{}—to further improve alignment performance. Finally, we aim to develop new self-guidance mechanisms for preference optimization and explore how \TheName{} can be leveraged to enhance the reasoning capabilities of large language models (LLMs).
    
    \begin{figure*}[t]
    \centering

    \begin{tabular}{cc}
      \includegraphics[width=0.45\linewidth]{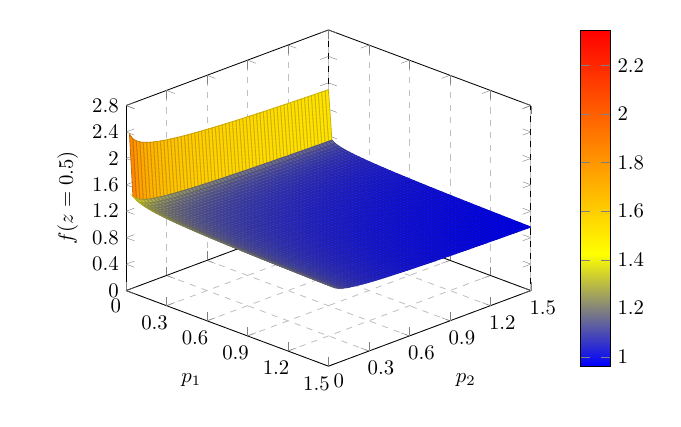} & 
      \includegraphics[width=0.45\linewidth]{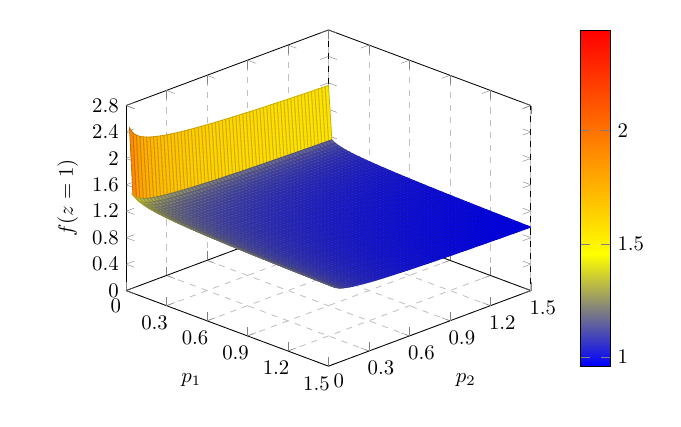}  \\  
     (a) $z=0.5$ & (b) $z = 1.0 $  \\
      \includegraphics[width=0.45\linewidth]{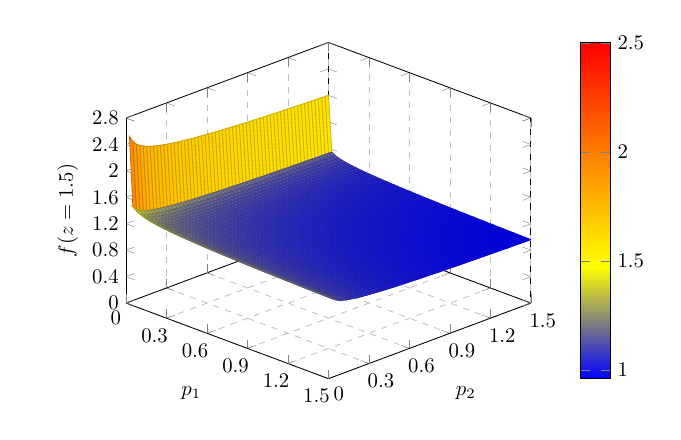} & 
      \includegraphics[width=0.45\linewidth]{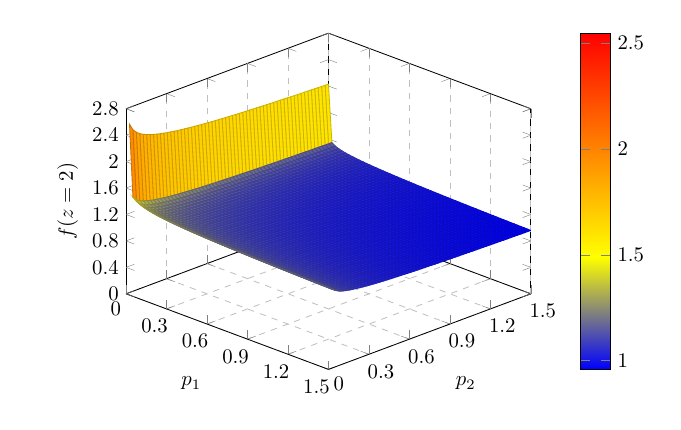} \\ 
     $z = 1.5 $ & (d) $z = 2.0 $ \\   
    \end{tabular}  
    \caption{Visual representation of the function $f(z)$ landscape.}
    \label{app:fig:fzs}
    \end{figure*}
    

\begin{figure*}[t]
    \centering
    \begin{tabular}{cc}
      \includegraphics[width=0.45\linewidth]{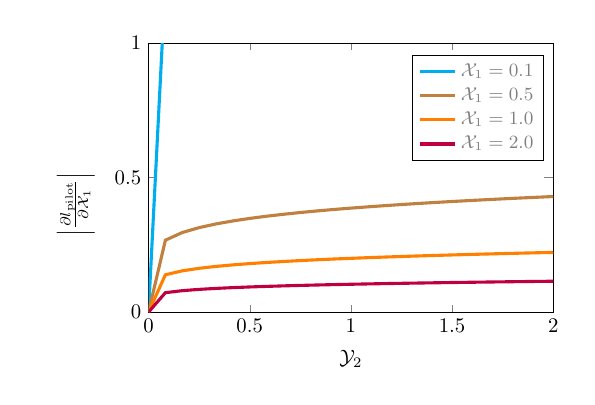}  &  
      \includegraphics[width=0.45\linewidth]{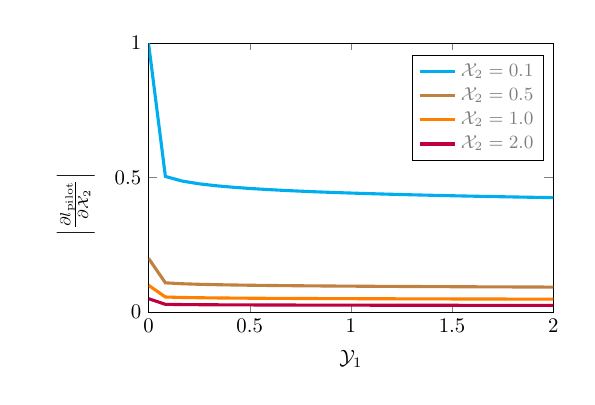}  \\
      (a) & (b)  \\
    \end{tabular}  
    \caption{Visual representation of the functions $\left\lvert \frac{\partial{l_\text{pilot}}}{\partial \mathcal{X}_1} \right\rvert$
     and $\left\lvert \frac{\partial{l_\text{pilot}}}{\partial \mathcal{X}_2} \right\rvert$ at selected fixed points.}
    \label{app:fig:partial-x}
    \end{figure*}

\end{document}